\documentclass{article}

\PassOptionsToPackage{numbers, compress}{natbib}
\usepackage[final]{nips_2017}

\usepackage[utf8]{inputenc} % allow utf-8 input
\usepackage[T1]{fontenc}    % use 8-bit T1 fonts
\usepackage{hyperref}       % hyperlinks
\usepackage{url}            % simple URL typesetting
\usepackage{booktabs}       % professional-quality tables
\usepackage{amsfonts}       % blackboard math symbols
\usepackage{nicefrac}       % compact symbols for 1/2, etc.
\usepackage{microtype}      % microtypography
\usepackage{amsmath}
\usepackage{ amssymb }  %to use \triangleq
\usepackage{amsthm}
\usepackage{verbatim}   % mult-line (block) comments
\usepackage{bbm}
\usepackage{color}

\DeclareMathOperator*{\argmax}{arg\,max}  % in your preamble
\usepackage{complexity}     % to use poly
\usepackage{natbib}
\usepackage{graphicx}
\renewcommand{\E}{\mathbb{E}}
\newcommand{\Prob}{\mathbb{P}}

\newtheorem{theorem}{Theorem}

\newtheorem{lemma}{Lemma}
\newtheorem{proposition}{Proposition}

\title{Improving the Expected Improvement Algorithm}

\author{
  Chao Qin\quad Diego Klabjan\quad Daniel Russo\\  
  Northwestern University\\
  \texttt{chaoqin2019@u.northwestern.edu}\\
  \texttt{d-klabjan@northwestern.edu}\\
  \texttt{dan.joseph.russo@gmail.com}
}

\begin{document}

\maketitle

\begin{abstract}
	The expected improvement (EI) algorithm is a popular strategy for information collection in optimization under uncertainty. The algorithm is widely known to be too greedy, but nevertheless enjoys wide use due to its simplicity and ability to handle uncertainty and noise in a coherent decision theoretic framework. To provide rigorous insight into EI, we study its properties in a simple setting of Bayesian optimization where the domain consists of a finite grid of points. This is the so-called best-arm identification problem, where the goal is to allocate measurement effort wisely to confidently identify the best arm using a small number of measurements. In this framework, one can show formally that EI is far from optimal. To overcome this shortcoming, we introduce a simple modification of the expected improvement algorithm. Surprisingly, this simple change results in an algorithm that is asymptotically optimal for Gaussian best-arm identification problems, and provably outperforms standard EI by an order of magnitude.  
\end{abstract}

\section{Introduction}

Recently Bayesian optimization has received much attention in the machine learning community \cite{DBLP:journals/pieee/ShahriariSWAF16}. This literature studies the problem of maximizing an unknown black-box objective function by collecting noisy measurements of the function at carefully chosen sample points. At first a prior belief over the objective function is prescribed, and then the statistical model is refined sequentially as data are observed. \textsl{Expected improvement (EI)} \cite{Jones1998} is one of the most widely-used Bayesian optimization algorithms. It is a greedy improvement-based heuristic that samples the point offering greatest expected improvement over the current best sampled point. EI is simple and readily implementable, and it offers reasonable performance in practice. 

Although EI is reasonably effective, it is too greedy, focusing nearly all sampling effort near the estimated optimum and gathering too little information about other regions in the domain. This phenomenon is most transparent in the simplest setting of Bayesian optimization where the function's domain is  a finite grid of points. This is the problem of best-arm identification (BAI) \cite{DBLP:conf/colt/AudibertBM10} in a multi-armed bandit. The player sequentially selects arms to measure and observes noisy reward samples with the hope that a small number of measurements enable a confident identification of the best arm. Recently \citet{doi:10.1287/opre.2016.1494} studied the performance of EI in this setting. His work focuses on a link between EI and another algorithm known as the optimal computing budget allocation \cite{chen2000simulation}, but his analysis reveals EI allocates a vanishing  proportion of samples to suboptimal arms as the total number of samples grows. Any method with this property will be far from optimal in BAI problems \cite{DBLP:conf/colt/AudibertBM10}. 

In this paper, we improve the EI algorithm dramatically through a simple modification. The resulting algorithm, which we call  \emph{top-two expected improvement} (TTEI), combines the top-two sampling idea of \citet{russo2016simple} with a careful change to the improvement-measure used by EI.  We show that this simple variant of EI achieves strong asymptotic optimality properties in the BAI problem, and benchmark the algorithm in simulation experiments. 

Our main theoretical contribution is a complete characterization of the asymptotic proportion of samples TTEI allocates to each arm as a function of the true (unknown) arm means. These particular sampling proportions have been shown to be optimal from several perspectives \cite{chernoff1959, jennison1982asymptotically,glynn2004large, russo2016simple, GarivierK16}, and this enables us to establish two different optimality results for TTEI. The first concerns the rate at which the algorithm gains confidence about the identity of the optimal arm as the total number of samples collected grows. Next we study the so-called fixed confidence setting, where the algorithm is able to stop at any point and return an estimate of the optimal arm. We show that when applied with the stopping rule of \citet{GarivierK16}, TTEI essentially minimizes the expected number of samples required among all rules obeying a constraint on the probability of incorrect selection.

One undesirable feature of our algorithm is its dependence on a tuning parameter. Our theoretical results precisely show the impact of this parameter, and reveal a surprising degree of robustness to its value. It is also easy to design methods that adapt this parameter over time to the optimal value, and we explore one such method in simulation. Still, removing this tuning parameter is an interesting direction for future research.

\paragraph{Further related literature.}
Despite the popularity of EI, its theoretical properties are not well studied. A notable exception is the work of \citet{journals/jmlr/Bull11}, who studies a global optimization problem and provides a convergence rate for EI's expected loss. However, it is assumed that the observations are noiseless. 
Our work also relates to  a large number of recent machine learning papers that try to characterize the sample complexity of the best-arm identification problem \cite{Even-dar02pacbounds,Mannor04thesample,DBLP:conf/colt/AudibertBM10,NIPS2012_4640,pmlr-v28-karnin13,pmlr-v35-jamieson14,DBLP:conf/ciss/JamiesonN14, pmlr-v30-Kaufmann13, pmlr-v35-kaufmann14, JMLR:v17:kaufman16a}. Despite substantial progress, matching asymptotic upper and lower bounds remained elusive in this line of work. Building on older work in statistics \cite{chernoff1959, jennison1982asymptotically} and simulation optimization \cite{glynn2004large}, recent work of \citet{GarivierK16} and \citet{russo2016simple} characterized the optimal sampling proportions. Two notions of asymptotic optimality are established: \textsl{sample complexity in the fixed confidence setting} and \textsl{rate of posterior convergence}. \citet{GarivierK16} developed two sampling rules designed to closely track the asymptotic optimal proportions and showed that, when combined with a stopping rule motivated by \citet{chernoff1959}, this sampling rule minimizes the expected number of samples required to guarantee a vanishing threshold on the probability of incorrect selection is satisfied. \citet{russo2016simple} independently proposed three simple Bayesian algorithms, and proved that each algorithm attains the optimal rate of posterior convergence.  TTEI proposed in this paper is conceptually most similar to the top-two value sampling of \citet{russo2016simple}, but it is more computationally efficient.

\subsection{Main Contributions}
As discussed below, our work makes both theoretical and algorithmic contributions.
\begin{description}
	\item[Theoretical:] Our main theoretical contribution is Theorem \ref{thm: sampling proportions}, which establishes that TTEI--a simple modification to a popular Bayesian heuristic--converges to the known optimal asymptotic sampling proportions. It is worth emphasizing that, unlike recent results for other top-two sampling algorithms \cite{russo2016simple}, this theorem establishes that the expected time to converge to the optimal proportions is finite, which we need to establish optimality in the fixed confidence setting. Proving this result required substantial technical innovations. Theorems \ref{thm: sc1} and \ref{thm: sc2} are additional theoretical contributions. These mirror results in 	\cite{russo2016simple} and \cite{GarivierK16}, but we extract minimal conditions on sampling rules that are sufficient to guarantee the two notions of optimality studied in these papers.  
	\item[Algorithmic:] On the algorithmic side, we substantially improve a widely used algorithm. TTEI can be easily implemented by modifying existing EI code, but, as shown in our experiments, can offer an order of magnitude improvement. A more subtle point involves the advantages of TTEI over algorithms that are designed to directly target convergence on the asymptotically optimal proportions. In the experiments, we show that TTEI substantially \emph{outperforms an oracle sampling rule} whose sampling proportions directly track the asymptotically optimal proportions. This phenomenon should be explored further in future work, but suggests that by carefully reasoning about the value of information TTEI accounts for important factors that are washed out in asymptotic analysis. Finally--as discussed in the conclusion--although we focus on uncorrelated priors we believe our method can be easily extended to more complicated problems like that of best-arm
	 identification in linear bandits \cite{soare2014best}. 
\end{description}
\section{Problem Formulation}

Let $A=\{1,\ldots,k\}$ be the set of arms. The reward $Y_{n,i}$ of arm $i \in A$ at time $n\in \mathbb{N}$ follows a normal distribution $N(\mu_i,\sigma^2)$ with common known variance $\sigma^2$, but unknown mean $\mu_i$. At each time $n=1,2,\ldots$, an arm $I_n\in A$ is measured, and the corresponding noisy reward $Y_{n,I_n}$ is observed. The objective is to allocate measurement effort wisely in order to confidently identify the arm with highest mean using a small number of measurements. %This is the so called (Gaussian) best-arm identification problem.
We assume that $\mu_{1}>\mu_{2}>\ldots >\mu_{k}$, i.e., the arm-means are unique and arm 1 is the best arm.   Our analysis takes place in a \emph{frequentist setting}, in which the true means $(\mu_{1},\ldots,\mu_{k})$ are fixed but unknown. The algorithms we study, however, are Bayesian,  in the sense that they begin with prior over the arm means and update the belief to form a posterior distribution as evidence is gathered.

\paragraph{Prior and Posterior Distributions.}

The sampling rules studied in this paper begin with a normally distributed prior over the true mean of each arm $i\in A$ denoted by $N(\mu_{1,i}, \sigma_{1,i}^2)$, and update this to form a posterior distribution as observations are gathered. 
%In particular, while we have formulated a frequentist problem where the true means $(\mu_{1},\ldots,\mu_{k})$ are fixed, the algorithm behaves as if $\mu_{i}\sim N(\mu_{1,i}, \sigma^{2}_{1,i})$. 
By conjugacy, the posterior distribution after observing the sequence  $(I_1,Y_{1,I_1},\ldots,I_{n-1},Y_{n-1,I_{n-1}})$ 
is also a normal distribution denoted by $N(\mu_{n,i}, \sigma_{n,i}^2)$. The posterior mean and variance can be calculated using the following recursive equations: %\cite{de1970optimal},
\begin{equation*}
\mu_{n+1,i} = 
\begin{cases}
(\sigma_{n,i}^{-2}\mu_{n,i} + \sigma^{-2}Y_{n,i})/(\sigma_{n,i}^{-2} + \sigma^{-2}) & \text{if } I_{n} = i,\\
\mu_{n,i}, & \text{if } I_{n}\neq i,
\end{cases} 
\end{equation*}
and
\[
\sigma_{n+1,i}^2 = 
\begin{cases}
1/(\sigma_{n,i}^{-2} + \sigma^{-2}) & \text{if } I_{n} = i,\\
\sigma_{n,i}^2, & \text{if } I_{n}\neq i.
\end{cases} .
\]
We denote the posterior distribution over the vector of arm means by 
\[ 
\Pi_{n} = N(\mu_{n,1}, \sigma_{n,1}^2) \otimes N(\mu_{n,2}, \sigma_{n,2}^2) \otimes \cdots \otimes N(\mu_{n,k}, \sigma_{n,k}^2)
\]
and let $\theta=(\theta_1,\ldots,\theta_k)$. For example, with this notation  
\[ 
\E_{\theta \sim \Pi_n}\left[ \sum_{i\in A} \theta_i \right] = \sum_{i\in A}\mu_{n,i}.
\]
The posterior probability assigned to the event that arm $i$ is optimal is
\begin{equation}\label{eq: optimal action probs}
\alpha_{n,i} \triangleq \Prob_{\theta \sim \Pi_n}\left( \theta_i > \max_{j\neq i} \theta_j\right).
\end{equation}
To avoid confusion, we use $\theta=(\theta_1,\ldots,\theta_k)$ to denote a random vector of arm means drawn from the algorithm's posterior $\Pi_n$, and $\mu=(\mu_1,\ldots,\mu_k)$ to denote the vector of true arm means.  

\paragraph{Two notions of asymptotic optimality.}
Our first notion of optimality relates to the rate of posterior convergence. As the number of observations grows, one hopes that the posterior distribution definitively identifies the true best arm, in the sense that the posterior probability $1-\alpha_{n,1}$ assigned by the event that a different arm is optimal tends to zero. By sampling the arms intelligently, we hope this probability can be driven to zero as rapidly as possible. 
We will see that under TTEI the posterior probability tends to zero at an exponential rate, and so following \citet{russo2016simple}, we aim to maximize the exponent governing the rate of decay, effectively solving the optimization problem
\[
\underset{\text{sampling rules}}{\min} \limsup_{n\to \infty} \,\, \frac{1}{n} \log\left( 1-\alpha_{n,1}\right). 
\]
The second setting we consider is often called the ``fixed confidence'' setting. Here, the agent is allowed at any point to stop gathering samples and return an estimate of the identity of the optimal. In addition to the sampling rule TTEI, we require a stopping rule that selects a time $\tau$ at which to stop, and decision rule that returns an estimate $\hat{i}_{\tau}$ of the optimal arm based on the first $\tau$  observations.  We consider minimizing the average number of observations $\E[\tau]$ required by an algorithm guaranteeing a vanishing probability $\delta$ of incorrect identification, i.e., $\Prob(\hat{i}_{\tau} \neq 1)\leq \delta$. Following \citet{GarivierK16}, the number of samples required scales with $\log(1/\delta)$, and so we aim to minimize 
\[ 
\limsup_{\delta \to 0} \frac{\E[\tau]}{\log(1/\delta)}
\]
among algorithms with probability of error no more than $\delta$. In this setting, we study the performance of EI when combined with the stopping rule studied by \citet{chernoff1959} and \citet{GarivierK16}.  

\section{Sampling Rules}\label{sec: algos}
In this section, we first introduce the expected improvement algorithm, and point out its weakness. Then a simple variant of the expected improvement algorithm is proposed. 
Both algorithms make calculations using function $f(x) = x\Phi(x) + \phi(x)$ where $\Phi(\cdot)$ and $\phi(\cdot)$ are the CDF and PDF of the standard normal distribution. One can show that as $x\to \infty$, $\log f(-x) \sim -x^2/2$, and so $f(-x) \approx e^{-x^2 /2 }$ for very large $x$. One can also show that $f$ is an increasing function. 

\paragraph{Expected Improvement.}
%To select an arm at each time, we induce myopic heuristics known as \textsl{acquisition function} $\alpha_n$ that leverage the uncertainty in the posterior to guide exploration. 
\textsl{Expected improvement} \cite{Jones1998} is a simple improvement-based sampling rule. The EI algorithm favors the arm that offers the largest amount of improvement upon a target. The EI algorithm measures the arm $I_{n} = \argmax_{i\in A} v_{n,i}$ where $v_{n,i}$ is the EI value of arm $i$ at time $n$. Let $I_n^* = \argmax_{i\in A}\mu_{n,i}$ denote the arm with largest posterior mean at time $n$. The EI value of arm $i$ at time $n$ is defined as
\[
v_{n,i} \triangleq \E_{\theta\sim \Pi_n} \left[\left(\theta_i - \mu_{n,I_n^*}\right)^+\right].
\]
where $x^+ = \max\{x,0\}$. 
The above expectation can be computed analytically as follows,
\begin{align*}
v_{n,i}
=&\left(\mu_{n,i}-\mu_{n,I_n^*}\right)  \Phi\left(\frac{\mu_{n,i}-\mu_{n,I_n^*}}{\sigma_{n,i}}\right)   + \sigma_{n,i} \phi\left(\frac{\mu_{n,i}-\mu_{n,I_n^*}}{\sigma_{n,i}}\right)
=\sigma_{n,i} f\left( \frac{\mu_{n,i}-\mu_{n,I_n^*}}{\sigma_{n,i}} \right).
\end{align*}
The EI value $v_{n,i}$ measures the potential of arm $i$ to improve upon the largest posterior mean $\mu_{n,I_n^*}$ at time $n$. Because $f$ is an increasing function, $v_{n,i}$ is increasing in both the posterior mean $\mu_{n,i}$ and posterior standard deviation $\sigma_{n,i}$.

\paragraph{Top-Two Expected Improvement.}
The EI algorithm can have very poor performance for selecting the best arm. Once it finds a particular arm with reasonably high probability to be the best, it allocates nearly all future samples to this arm at the expense of measuring other arms. Recently \citet{doi:10.1287/opre.2016.1494} showed that EI only allocates $\mathcal{O}(\log n)$ samples to suboptimal arms asymptotically. This is a severe shortcoming, as it means $n$ must be extremely large before the algorithm has enough samples from suboptimal arms to reach a confident conclusion. 

To improve the EI algorithm, we build on the top-two sampling idea in Russo \cite{russo2016simple}. The idea is to identify in each period the two ``most promising'' arms based on current observations, and randomize to choose which to sample. A tuning parameter $\beta \in (0,1)$ controls the probability assigned to the ``top'' arm. A naive top-two variant of EI would identify the two arms with largest EI value, and flip a $\beta$--weighted coin to decide which to measure. However, one can prove that this algorithm is not optimal for any choice of $\beta$. Instead, what we call the top-two expected improvement algorithm uses a novel modified EI criterion which more carefully accounts for the decision-maker's uncertainty when deciding which arm to sample.  

For $i, j \in A$, define $v_{n,i,j} \triangleq \E_{\theta \sim \Pi_{n}}\left[ (\theta_i - \theta_j)^{+} \right]$. This measures the expected magnitude of improvement arm $i$ offers over arm $j$, but unlike the typical EI criterion, this expectation integrates over the uncertain quality of \emph{both arms}. This measure can be computed analytically as
\[
v_{n,i,j} = \sqrt{\sigma_{n,i}^2+\sigma_{n,j}^2} f\left(\frac{\mu_{n,i}-\mu_{n,j}}{\sqrt{\sigma_{n,i}^2+\sigma_{n,j}^2}}\right).
\]

TTEI depends on a tuning parameter $\beta > 0$, set to $1/2$ by default. 
%\textbf{TTEI measures each arm once in the first $k$ periods.} 
With probability $\beta$, TTEI measures the arm $I_n^{(1)}$ by optimizing the EI criterion, and otherwise it measures an alternative $I_n^{(2)}$ that offers the largest expected improvement on the arm $I_n^{(1)}$.  Formally, TTEI measures the arm 
\begin{equation*}
I_n = 
\begin{cases}
I_n^{(1)} = \argmax_{i\in A} v_{n,i}, & \text{with probability } \beta,\\
I_n^{(2)} = \argmax_{i\in A} v_{n,i, I_{n}^{(1)}}, & \text{with probability } 1-\beta.
\end{cases}
\end{equation*}
Note that $v_{n,i, i}=0$, which implies $I_n^{(2)}\neq I_n^{(1)}$.

We notice that TTEI with $\beta=1$ is the standard EI algorithm. Comparing to the EI algorithm, TTEI with $\beta\in(0,1)$ allocates much more measurement effort to suboptimal arms. We will see that TTEI allocates $\beta$ proportion of samples to the best arm asymptotically, and it uses the remaining $1-\beta$ fraction of samples for gathering evidence against each suboptimal arm.

\section{Convergence to Asymptotically Optimal Proportions}
For all $i\in A$ and $n\in\mathbb{N}$, we define $T_{n,i} \triangleq \sum_{\ell=1}^{n-1} \mathbf{1}\{I_\ell = i\}$ to be the number of samples of arm $i$ before time $n$.
We will show that under TTEI with parameter $\beta$, $ \lim_{n\to \infty} T_{n,1}/n= \beta$. That is, the algorithm asymptotically allocates $\beta$ proportion of the samples to true best arm. Dropping for the moment questions regarding the impact of this tuning parameter, let us consider the optimal asymptotic proportion of effort to allocate to each f the $k-1$ remaining arms. It is known that the optimal proportions are given by the unique vector $(w^{\beta}_2,\cdots, w^{\beta}_k)$ satisfying,  $\sum_{i=2}^{k} w_{i}^{\beta} = 1-\beta$ and  
\begin{equation}\label{eq: optimal proportions}
\frac{(\mu_{2}-\mu_{1})^2}{1/w_{2}^\beta+1/\beta} =\ldots = \frac{(\mu_{k}-\mu_{1})^2}{1/w_{k}^\beta+1/\beta}.
\end{equation}
We set $w^{\beta}_1 = \beta$, so $w^{\beta}=\left(w^{\beta}_1,\ldots, w^{\beta}_k\right)$ encodes the sampling proportions of each arm.

To understand the source of equation \eqref{eq: optimal proportions}, imagine that over the first $n$ periods each arm $i$ is sampled exactly $w_{i}^{\beta} n$ times, and let $\hat{\mu}_{n,i}  \sim N\left(\mu_i, \frac{\sigma^2}{w_{i}^{\beta} n}\right)$ denote the empirical mean of arm $i$. Then 
\[ 
\hat{\mu}_{n,1} - \hat{\mu}_{n,i} \sim N\left(\mu_{1}-\mu_{i}, \tilde{\sigma}_{i}^2\right) \qquad \text{where} \qquad \tilde{\sigma}^2_{i}=\frac{\sigma^2}{n/\beta + n/w^{\beta}_i}.
\]
The probability $\hat{\mu}_{n,1} -\hat{\mu}_{n,i} \leq 0$--leading to an incorrect estimate of the  arm with highest mean--is $\Phi\left((\mu_i - \mu_1) / \tilde{\sigma}_i\right)$ where $\Phi$ is the CDF of the standard normal distribution. Equation \eqref{eq: optimal proportions} is equivalent to requiring $(\mu_1 - \mu_i) / \tilde{\sigma}_i$ is equal for all arms $i$, so the probability of falsely declaring $\mu_{i} \geq \mu_1$ is equal for all $i\neq 1$. In a sense, these sampling frequencies equalize the evidence against each suboptimal arm. These proportions appeared first in the machine learning literature in \cite{russo2016simple, GarivierK16}, but appeared much earlier in the statistics literature in \cite{jennison1982asymptotically}, and separately in the simulation optimization literature in \cite{glynn2004large}. As we will see in the next section, convergence to this allocation is a necessary condition for both notions of optimality considered in this paper. 

Our main theoretical contribution is the following theorem, which establishes that under TTEI sampling proportions converge to the proportions $w^{\beta}$ derived above. Therefore, while the sampling proportion of the optimal arm is controlled by the tuning parameter $\beta$, the remaining $1-\beta$ fraction of measurement is optimally distributed among the remaining $k-1$ arms. One of our results requires more than convergence to $w^{\beta}$ with probability 1, but a sense in which the expected time until convergence is finite. To make this precise, we introduce a time after which for each arm, both its empirical mean and empirical proportion are accurate. Specifically, given $\beta\in(0,1)$ and $\epsilon>0$, we define
\begin{equation}\label{eq: time to reach proportions}
T^\epsilon_\beta\triangleq\inf \left\{N\in\mathbb{N} \,:\, |\mu_{n,i}-\mu_i|\leq\epsilon \text{ and } |T_{n,i}/n - w^\beta_i|\leq\epsilon,\forall i\in A \text{ and } n\geq N \right\}.
\end{equation}
If $T_{n,i}/n \to w^{\beta}_i$ with probability 1, then by the law of large numbers $\Prob(T^{\epsilon}_\beta < \infty) =1$ for every $\epsilon>0$. Such a result was established for other top-two sampling algorithms in \cite{russo2016simple}. To establish optimality in the ``fixed confidence setting'', we need to prove in addition that $\E[T^{\beta}_\epsilon]<\infty$ for all $\epsilon>0$, which requires substantial new technical innovations. 
\begin{theorem}\label{thm: sampling proportions}
	If TTEI is applied with parameter $\beta \in (0,1)$, $\E[T^\epsilon_\beta]<\infty$ for any $\epsilon>0$. Therefore,
	\[
	\lim_{n \to \infty} \frac{T_{n,i}}{n} = w^{\beta}_{i} \qquad \forall i \in A.
	\] 
\end{theorem}

\subsection{Problem Complexity Measure}
Given $\beta\in(0,1)$, define the problem complexity measure
\[
\Gamma_\beta^* \triangleq \frac{(\mu_{2}-\mu_{1})^2}{2\sigma^2\left(1/w_{2}^{\beta}+1/\beta\right)} =\ldots = \frac{(\mu_{k}-\mu_{1})^2}{2\sigma^2\left(1/w_{k}^{\beta}+1/\beta\right)},
\]
which is a function of the true arm means and variances. This will be the exponent governing the rate of posterior convergence, and also characterizing the average number of samples in the fixed confidence stetting. The optimal exponent comes from maximizing over $\beta$. Let us define $\Gamma^* = \max_{\beta\in(0,1)}\Gamma_\beta^*$ and $\beta^* = \argmax_{\beta\in(0,1)} \Gamma_\beta^*$ and set 
\[
w^*= w^{\beta^*} = \left(\beta^*, w_{2}^{\beta^*},\ldots,w_{k}^{\beta^*}\right).
\]
Russo \cite{russo2016simple} has proved that for $\beta\in(0,1)$,
$\Gamma^*_{\beta} \geq \Gamma^*/\max\left\{\frac{\beta^*}{\beta},\frac{1-\beta^*}{1-\beta}\right\},$ and therefore $\Gamma_{1/2}^*\geq \Gamma^*/2$. This demonstrates a surprising degree of robustness to $\beta$. In particular, $\Gamma_{\beta}$ is close to $\Gamma^*$ if $\beta$
is adjusted to be close to $\beta^*$, and the choice of $\beta=1/2$ always yields a 2-approximation to $\Gamma^*$. 

\section{Implied Optimality Results}
This section establishes formal optimality guarantees for TTEI. Both results, in fact, hold for any algorithm satisfying the conclusions of Theorem \ref{thm: sampling proportions}, and is therefore one of broader interest. 

\subsection{Optimal Rate of Posterior Convergence}
We first provide upper and lower bounds on the exponent governing the rate of posterior convergence. The same result has been has been proved in \citet{russo2016simple} for bounded correlated priors. We use different proof techniques to prove the following result for uncorrelated Gaussian priors. 

This theorem shows that no algorithm can attain a rate of posterior convergence faster than $e^{-\Gamma^* n}$ and that this is attained by any algorithm that, like TTEI with optimal tuning parameter $\beta^*$, has asymptotic sampling ratios $(w^*_1,\ldots,w^*_k)$. The second part implies TTEI with parameter $\beta$ attains convergence rate $e^{-n \Gamma^*_\beta}$ and that it is optimal among sampling rules that allocation $\beta$--fraction of samples to the optimal arm. Recall that, without loss of generality, we have assumed arm $1$ is the arm with true highest mean $\mu_1 = \max_{i\in A} \mu_i$. We will study the posterior mass $1-\alpha_{n,1}$ assigned to the event that some other has the highest mean.  
\begin{theorem}[Posterior Convergence - Sufficient Condition for Optimality]	\label{thm: sc1}
	The following properties hold with probability 1:
	\begin{enumerate}
		\item Under any allocation rule satisfying $T_{n,i}/n \to w_i^*$ for each $i \in A$,
		\[
		\lim_{n\to\infty} \,\, -\frac{1}{n}\log\left( 1-\alpha_{n,1} \right) = \Gamma^*.
		\]       
		Under any sampling rule,
		\[
		\limsup_{n\to\infty} \,\,-\frac{1}{n}\log(1-\alpha_{n,1}) \leq \Gamma^*.
		\]
		\item For $\beta \in (0,1)$, under any allocation rule satisfying $T_{n,i}/n \to w^{\beta}_i$ for each $i\in A$,
		\[
		\lim_{n\to\infty} -\frac{1}{n}\log(1-\alpha_{n,1}) = \Gamma_\beta^*.
		\]   
		Under any sampling rule satisfying $T_{n,1}/n \to \beta$,  
		\[
		\limsup_{n\to\infty} \,\, -\frac{1}{n}\log(1-\alpha_{n,1}) \leq \Gamma_\beta^*.
		\]
	\end{enumerate}
\end{theorem}

This result reveals that when the tuning parameter $\beta$ is set optimally to $\beta^*$, TTEI attains the optimal rate of posterior convergence. Since $\Gamma^*_{1/2}\geq \Gamma^* /2$, when $\beta$ set to the default value $1/2$, the exponent governing the convergence rate of TTEI is at least half of the optimal one.

\subsection{Optimal Average Sample Size} \label{section_stopping}

\paragraph{Chernoff's Stopping Rule.}
In the fixed confidence setting, besides an efficient sampling rule, a player also needs to design an intelligent stopping rule. This section introduces a stopping rule proposed by \citet{chernoff1959} and studied recently by \citet{GarivierK16}. This stopping rule makes use of the Generalized Likelihood Ratio statistic, which depends on the current maximum likelihood estimates of all unknown means. For each arm $i\in A$, the maximum likelihood estimate of its unknown mean $\mu_i$ at time $n$ is its empirical mean $\hat{\mu}_{n,i} = T_{n,i}^{-1} \sum_{\ell=1}^{n-1} \mathbf{1}\{I_\ell = i\}Y_{\ell,I_\ell}$. If $T_{n,i}=0$, we set $\hat{\mu}_{n,i} = 0$.  For arms $i,j\in A$, if $\hat{\mu}_{n,i} \geq \hat{\mu}_{n,j}$, the Generalized Likelihood Ratio statistic $Z_{n,i,j}$ has the following explicit expression for Gaussian noise distributions:
\[
Z_{n,i,j} \triangleq T_{n,i}d(\hat{\mu}_{n,i},\hat{\mu}_{n,i,j}) + T_{n,j}d(\hat{\mu}_{n,j},\hat{\mu}_{n,i,j})
\]
where $d(x,y) \triangleq (x-y)^2/(2\sigma^2)$ is the KL-divergence between two normal distributions $N(x,\sigma^2)$ and $N(y,\sigma^2)$, and $\hat{\mu}_{n,i,j}$ is a weighted average of the empirical means of arms $i,j$ defined as
\[
\hat{\mu}_{n,i,j} \triangleq \frac{T_{n,i}}{T_{n,i}+T_{n,j}}\hat{\mu}_{n,i} + \frac{T_{n,j}}{T_{n,i}+T_{n,j}}\hat{\mu}_{n,j}.
\]
On the other hand, if $\hat{\mu}_{n,i}<\hat{\mu}_{n,j}$, then $Z_{n,j,i}$ is well-defined as above, and $Z_{n,i,j} = -Z_{n,j,i}\leq 0$ (if $T_{n,i}=T_{n,j}=0$, we let $Z_{n,i,j}=Z_{n,j,i}=0$). Given a target confidence $\delta\in(0,1)$, to ensure that one arm is better than the others with probability at least $1-\delta$, we use the stopping time
\[
\tau_{\delta} \triangleq \inf\left\{n\in\mathbb{N} \,:\, Z_n \triangleq \max_{i\in A}\min_{j\in A\setminus\{i\}} Z_{n,i,j} > \gamma_{n,\delta} \right\}
\]
where $\gamma_{n,\delta}>0$ is an appropriate threshold. By definition, we known that $\min_{j\in A\setminus\{i\}} Z_{n,i,j}$ is nonnegative if and only if $\hat{\mu}_{n,i}\geq \hat{\mu}_{n,j}$ for all $j\in A\setminus\{i\}$. Hence, whenever $\hat{I}_n^*\triangleq\argmax_{i\in A} \hat{\mu}_{n,i}$ is unique, $Z_n= \min_{j\in A\setminus \left\{\hat{I}_n^*\right\}}Z_{n,\hat{I}_n^*,j}$.

Next we introduce the exploration rate for normal bandit models that can ensure to identify the best arm with probability at least $1-\delta$. We use the following result given in Garivier and Kaufmann \cite{GarivierK16}.
\begin{proposition} [Garivier and Kaufmann \cite{GarivierK16} Proposition 12]
	\label{GK}
	Let $\delta \in (0,1)$ and $\alpha > 1$. For any normal bandit model, there exists a constant $C= C(\alpha,k)$ such that under any possible sampling rule, using the Chernoff's stopping rule with the threshold $\gamma_{n,\delta}^\alpha =\log(C n^\alpha/\delta)$ guarantees
	\[
	\mathbb{P}\left(\tau_\delta < \infty, \argmax_{i\in A}\hat{\mu}_{\tau_\delta,i}\neq 1\right) \leq \delta.
	\]
\end{proposition}

\paragraph{Sample Complexity.}
Garivier and Kaufmann \cite{GarivierK16} recently provided a general lower bound on the number of samples required in the fixed confidence setting. In particular, they show that for any normal bandit model, under any sampling rule and stopping time $\tau_\delta$ that guarantees a probability of error less than $\delta$,
\[
\liminf_{\delta\to0} \frac{\mathbb{E}[\tau_\delta]}{\log(1/\delta)} \geq \frac{1}{\Gamma^*}.
\]
Recall that $T^\epsilon_\beta$, defined in \eqref{eq: time to reach proportions}, is the first time after which the empirical means and empirical proportions are within  $\epsilon$ of their asymptotic limits. The next result provides a condition in terms of $T^\epsilon_\beta$ that is sufficient to guarantees optimality in the fixed confidence setting.
\begin{theorem}[Fixed Confidence - Sufficient Condition for Optimality]	\label{thm: sc2}
	Let $\beta\in(0,1)$. Consider any sampling rule which, if applied with no stopping rule, satisfies $\mathbb{E}[T^{\epsilon}_\beta]<\infty$ for all $\epsilon>0$. Fix any $\alpha > 1$. Then if this sampling rule is applied with Chernoff's stopping rule with the threshold $\gamma_{n,\delta}^\alpha$, we have
	\[
	\limsup_{\delta\to0} \frac{\mathbb{E}[\tau_\delta]}{\log(1/\delta)} \leq \frac{\alpha}{\Gamma_\beta^*}.
	\]
\end{theorem}
Since $\alpha$ can be chosen to be arbitrarily close to 1, when $\beta = \beta^*$ the general lower bound on sample complexity of $1/\Gamma^*$ is essentially matched. In addition, when $\beta$ is set to the default value $1/2$ and $\alpha$ is taken to be arbitrarily close to 1, the sample complexity of TTEI combined with the Chernoff's stopping rule is at most twice the optimal sample complexity since $1/\Gamma_{1/2}^*\leq 2/\Gamma^*$.

\section{Numerical Experiments}
To test the empirical performances of TTEI, we conduct several numerical experiments. The first experiment compares the performance of TTEI with $\beta=1/2$ and EI. The second experiment compares the performances of different versions of TTEI, top-two Thompson sampling (TTTS) \cite{russo2016simple}, knowledge gradient (KG) \cite{frazier2008knowledge} and oracle algorithms that know the optimal proportions \textsl{a priori}. Each algorithm plays arm $i=1,\ldots,k$ exactly once at the beginning, and then prescribe a prior $N(Y_{i,i},\sigma^2)$ for unknown arm-mean $\mu_i$ where $Y_{i,i}$ is the observation from $N(\mu_i,\sigma^2)$. In both experiments, we fix the common known variance $\sigma^2=1$ and the number of arms $k=5$. We consider three instances $[\mu_1,\ldots,\mu_5] =[5,4,1,1,1], [5,4,3,2,1]$ and $[2,0.8,0.6,0.4,0.2]$. The optimal parameter $\beta^*$ equals 0.48, 0.45 and 0.35, respectively.

%In the first experiment, we consider the instance $(\mu_1,\ldots,\mu_5) = (5,4,3,2,1)$. 
Recall that $\alpha_{n,i}$, defined in \eqref{eq: optimal action probs},  denotes the posterior probability that arm $i$ is optimal. Table \ref{table1} shows the average number of measurements required for the largest posterior probability being the best to reach a given confidence level $c$, i.e., $\max_i \alpha_{n,i} \geq c$.  The results in Table \ref{table1} are averaged over 100 trials. We see that TTEI with $\beta=1/2$ outperforms standard EI by an order of magnitude.
\begin{table}[h!]
\caption{Average number of measurements required to reach the confidence level $c=0.95$}
\label{table1}
	\centering  
	\begin{tabular}{lrr}  
		\hline
		& TTEI-1/2 & EI \\ \hline  
		$[5, 4, 1, 1, 1]$ & 14.60 & 238.50 \\
		$[5, 4, 3, 2, 1]$ & 16.72 & 384.73 \\
		$[2, .8, .6, .4, .2]$ & 24.39 & 1525.42 \\ \hline
	\end{tabular}
\end{table}

The second experiment compares the performance of different versions of TTEI, TTTS, KG, random sampling oracle (RSO) and tracking oracle (TO). The random sampling oracle draws a random arm in each round from the distribution $w^*$ encoding the asymptotically optimal proportions. The tracking oracle tracks the optimal proportions at each round. Specifically, the tracking oracle samples the arm with the largest ratio its optimal and empirical proportions. Two tracking algorithms proposed by Garivier and Kaufmann \cite{GarivierK16} are similar to this tracking oracle. TTEI with adaptive $\beta$ (aTTEI) works as follows: it starts with $\beta=1/2$ and updates $\beta=\hat{\beta}^*$ every 10 rounds where $\hat{\beta}^*$ is the maximizer of equation (\ref{eq: optimal proportions}) based on plug-in estimators for the unknown arm-means. 
%Besides $(5,4,3,2,1)$, we add two more examples $(2, 0.8, 0.6, 0.4, 0.2)$ and $(5, 4, 1, 1, 1)$. In the three examples, the parameter $\beta^*$ equals 0.45, 0.35 and 0.48, respectively. 
Table \ref{table2} shows the average number of measurements required for the largest posterior probability being the best to reach the confidence level $c=0.9999$. The results in Table \ref{table2} are averaged over 200 trials. We see that the performances of TTEI with adaptive $\beta$ and TTEI with $\beta^*$ are better than the performances of all other algorithms. We note that TTEI with adaptive $\beta$ substantially outperforms the tracking oracle.  

\begin{table}[h!]
	\centering  
			\caption{Average number of measurements required to reach the confidence level $c=0.9999$}
			\label{table2}
	\begin{tabular}{lrrrrrrr}  
		\hline
		& TTEI-1/2 & aTTEI  & TTEI-$\beta^*$ & TTTS-$\beta^*$ & RSO & TO & KG\\ \hline
		$[5, 4, 1, 1, 1]$ & 61.97 & 61.98 & {\bf 61.59} & 62.86 & 97.04  & 77.76 & 75.55 \\
		$[5, 4, 3, 2, 1]$ & 66.56 & {\bf 65.54} & 65.55 & 66.53 & 103.43 & 88.02 & 81.49 \\
		$[2, .8, .6, .4, .2]$ & 76.21 & 72.94 & {\bf 71.62} & 73.02 & 101.97 & 96.90 & 86.98 \\ \hline
	\end{tabular}
\end{table}

\section{Conclusion and Extensions to Correlated Arms}
%In this paper, we propose a simple modification of the EI algorithm and show this can generate enormous performance gains. Moreover, we show our algorithm satisfies strong asymptotic optimality properties and has strong performance in simulation for best-arm identification problems with uncorrelated priors and a small number of arms. 
We conclude by noting that while this paper thoroughly studies TTEI in the case of uncorrelated priors, we believe the algorithm is also ideally suited to problems with complex correlated priors and large sets of arms. In fact, the modified information measure $v_{n,i,j}$ was designed with an eye toward dealing with correlation in a  sophisticated way. In the case of a correlated normal distribution $N(\mu, \Sigma)$, one has
\[
v_{n,i,j}=\E_{\theta \sim N(\mu, \Sigma)}[ (\theta_i - \theta_j)^{+}] = \sqrt{\Sigma_{ii} + \Sigma_{jj} - 2 \Sigma_{ij}} f\left( \frac{\mu_{n,i}- \mu_{n,j} }{\sqrt{\Sigma_{ii} + \Sigma_{jj} - 2 \Sigma_{ij}}}  \right).  
\]
This closed form accommodates efficient computation. Here the term $\Sigma_{i,j}$ accounts for the correlation or similarity between arms $i$ and $j$. Therefore $v_{n, i,I^{(1)}_n}$ is large for arms $i$ that offer large potential improvement over $I_{n}^{(1)}$, i.e. those that (1) have large posterior mean, (2) have large posterior variance, and (3) are not highly correlated with arm $I_{n}^{(1)}$. As $I_{n}^{(1)}$ concentrates near the estimated optimum, we expect the third factor will force the algorithm to experiment in promising regions of the domain that are ``far'' away from the current-estimated optimum, and are under-explored under standard EI.

%\paragraph{Decision-theoretic stopping rules}

%\newpage
\bibliography{references}
\bibliographystyle{plainnat}

\newpage
\appendix

\section{Outline}
The appendix is organized as follows.
\begin{enumerate}
    \item Section \ref{notation} introduces some further notations required in the theoretical analysis.
    \item Section \ref{sufficient condition 1} is the proof of Theorem \ref{thm: sc1}, a sufficient condition in terms of optimal proportions $(w^\beta_1,\ldots,w^\beta_k)$ to guarantee the optimal rate of posterior convergence.
    \item Section \ref{sufficient condition 2} is the proof of Theorem \ref{thm: sc2}, a sufficient condition in terms of $T_\beta^\epsilon$ under which the optimality in the fixed confidence setting is achieved.
    \item Section \ref{preliminaries} provides several basic results which is used in the theoretical analysis of TTEI.
    \item Section \ref{sampling proportions} proves that TTEI satisfies the sufficient conditions for two notions of optimality, which immediately establishes Theorems \ref{thm: sampling proportions}.
\end{enumerate}

\section{Notation}\label{notation}
For notational convenience, we assume that sampling rules begin with an improper prior for each arm $i\in A$ with $\mu_{1,i}=0$ and $\sigma_{1,i} =\infty$. Consequently, if $T_{n,i}=\sum_{\ell=1}^{n-1} \mathbf{1}\{I_\ell = i\} = 0$, $\mu_{n,i} =\mu_{1,i}=0$ and $\sigma_{n,i}=\sigma_{1,i}=\infty$, and if $T_{n,i} > 0$,  
\[
\mu_{n,i} = \frac{1}{T_{n,i}} \sum_{\ell=1}^{n-1} 
\mathbf{1}\{I_\ell = i\}Y_{\ell,I_\ell} \quad\text{and}\quad \sigma_{n,i}^2 = \frac{\sigma^2}{T_{n,i}},
\]
so the posterior parameters are identical to the frequentist sample mean and variance under the observations collected so far.

We introduce some further notations. We define 
\[
\Delta_{\min} \triangleq \min_{i\neq j}|\mu_i-\mu_j| \quad\text{and}\quad \Delta_{\max} \triangleq \max_{i,j\in A}(\mu_i-\mu_j).
\]
Since the arm means are unique, we have $\Delta_{\min},\Delta_{\max} > 0$. In addition, we define
\[ 
\beta_{\min} \triangleq \min\{\beta, 1-\beta \} \quad\text{and}\quad \beta_{\max} \triangleq \max\{\beta, 1-\beta \}.
\]
Note that for $\beta\in(0,1)$, $\beta_{\min}>0$.

We introduce the filtration $(\mathcal{F}_{n} : n=1,2,\dots)$ where 
\[ 
\mathcal{F}_{n}=\Sigma(I_1, Y_{1,I_1},\cdots, I_{n}, Y_{n, I_{n}})
\]
is the sigma algebra generated by observations up to time $n$. For all $i\in A$ and $n\in\mathbb{N}$, define 
\[
\psi_{n,i} \triangleq \mathbb{P}(I_n = i |  \mathcal{F}_{n-1})  \quad\text{and}\quad \Psi_{n,i} \triangleq \sum_{\ell=1}^{n-1} \psi_{\ell, i}.
\]
Note that for all $i\in A$, $T_{1,i}=\Psi_{1,i}=0$. Both $T_{n,i}$ and $\Psi_{n,i}$ measure the effort allocated to arm $i$ up to period $n$.
%Letting $T_{n,i}$ be the number of evaluations allocated to arm $i$ before time $n$, we have $T_{n,i} = \sum_{\ell=1}^{n-1} \mathbf{1}\{I_\ell = i\}$ and $\sum_{i=1}^k T_{n,i} =n-1$. 

Finally, rather than use the notation $v_{n,i}$ and $v_{n,i,j}$ introduced in Section \ref{sec: algos} for the expected-improvement measures it is more convenient to work with the notation defined here. Set 
\[
v^{(1)}_{n,i} \equiv v_{n,i} \quad \forall i \in A
\] 
to be the expected improvement used in the identifying the first among in the top-two, and 
\[
v_{n,i}^{(2)} \equiv v_{n,i,I_{n}^{(1)}} \qquad \forall i \in A
\]
to be the second expected improvement measure where $I_n^{(1)}$ is the arm optimizing the first expected improvement measure.

\section{Proof of Theorem \ref{thm: sc1}} \label{sufficient condition 1}
To prove Theorem \ref{thm: sc1}, we first need to introduce the so-called Gaussian tail inequality.

\begin{lemma}
\label{tail}
Let $X\sim N(\mu,\sigma^2)$ and $c\geq 0$, then we have
\[
\frac{1}{\sqrt{2\pi}} e^{-(\sigma+c)^2/(2\sigma^2)}
\leq\mathbb{P}(X\geq\mu+c) 
\leq \frac{1}{2} e^{-c^2/(2\sigma^2)}.
\]
\end{lemma}

\begin{proof}
We first prove the upper bound.
\begin{align*}
\mathbb{P}(X\geq\mu+c) &= \int_{\mu+c}^{\infty} \frac{1}{\sqrt{2\pi\sigma^2}} e^{-(x-\mu)^2/(2\sigma^2)} dx\\
&= \int_{0}^{\infty} \frac{1}{\sqrt{2\pi\sigma^2}} e^{-(x+c)^2/(2\sigma^2)} dx\\
&\leq \int_{0}^{\infty} \frac{1}{\sqrt{2\pi\sigma^2}} e^{-(x^2+c^2)/(2\sigma^2)} dx\\
&=e^{-c^2/(2\sigma^2)} \int_{0}^{\infty} \frac{1}{\sqrt{2\pi\sigma^2}} e^{-x^2/(2\sigma^2)} dx\\
&=\frac{1}{2} e^{-c^2/(2\sigma^2)}.
\end{align*}

Next we prove the lower bound.
\begin{align*}
\mathbb{P}(X\geq\mu+c) &= \int_{\mu+c}^{\infty} \frac{1}{\sqrt{2\pi\sigma^2}} e^{-(x-\mu)^2/(2\sigma^2)} dx\\
&= \int_{0}^{\infty} \frac{1}{\sqrt{2\pi\sigma^2}} e^{-(x+c)^2/(2\sigma^2)} dx\\
&\geq \int_{0}^{\sigma} \frac{1}{\sqrt{2\pi\sigma^2}} e^{-(x+c)^2/(2\sigma^2)} dx\\
&\geq \int_{0}^{\sigma} \frac{1}{\sqrt{2\pi\sigma^2}} e^{-(\sigma+c)^2/(2\sigma^2)} dx\\
&=\frac{1}{\sqrt{2\pi}} e^{-(\sigma+c)^2/(2\sigma^2)}.
\end{align*}
\end{proof}

%By definition, $\alpha_{n,1} = \Prob_{\theta \sim \Pi_n}\left( \theta_1 > \max_{i\neq 1} \theta_i\right)$, so $1 - \alpha_{n,1} = \Prob_{\theta \sim \Pi_n} \left(\cup_{i\neq 1} (\theta_i\geq \theta_1)\right)$, and then we have $1-\alpha_{n,1} \geq \max_{i\neq 1} \Prob_{\theta \sim \Pi_n} \left(\theta_i\geq \theta_1\right)$. Recall that for each arm $i\in A$, the prior parameters are $\mu_{1,i}=0$ and $\sigma_{1,i} =\infty$. Consequently, if $T_{n,i}=\sum_{\ell=1}^{n-1} \mathbf{1}\{I_\ell = i\} = 0$, $\mu_{n,i} =\mu_{1,i}=0$ and $\sigma_{n,i}=\sigma_{1,i}=\infty$, and if $T_{n,i} > 0$,  
%\[
%\mu_{n,i} = \frac{1}{T_{n,i}} \sum_{\ell=1}^{n-1} 
%\mathbf{1}\{I_\ell = i\}Y_{\ell,I_\ell} \quad\text{and}\quad \sigma_{n,i}^2 = \frac{\sigma^2}{T_{n,i}}.
%\]

\paragraph{Proof of Theorem \ref{thm: sc1}.} %Notice that Theorem \ref{thm: sc1} is a sample-path result. We fix a sample path.
We let $\mathcal{I} =\{i\in A \,:\, \lim_{n\to\infty} T_{n,i} =\infty\}$ and $\overline{\mathcal{I}} = A\setminus \mathcal{I}$. Note that $\overline{\mathcal{I}}$ contains arms that are only sampled finite times.
First, suppose that $\overline{\mathcal{I}}$ is nonempty. For each $i\in A$, we define
\[
\mu_{\infty,i} \triangleq \lim_{n\to\infty} \mu_{n,i}  \quad\text{and}\quad \sigma_{\infty,i}^2 \triangleq \lim_{n\to\infty} \sigma_{n,i}^2.
\]
Recall that for each $i\in A$, an improper prior with $\mu_{1,i}=0$ and $\sigma_{1,i} =\infty$ is prescribed. Then if $T_{n,i}=\sum_{\ell=1}^{n-1} \mathbf{1}\{I_\ell = i\} = 0$, $\mu_{n,i} =\mu_{1,i}=0$ and $\sigma_{n,i}=\sigma_{1,i}=\infty$, and if $T_{n,i} > 0$.  
\[
\mu_{n,i} = \frac{1}{T_{n,i}} \sum_{\ell=1}^{n-1} 
\mathbf{1}\{I_\ell = i\}Y_{\ell,I_\ell} \quad\text{and}\quad \sigma_{n,i}^2 = \frac{\sigma^2}{T_{n,i}},
\]
Hence, for $i\in \mathcal{I}$, $\mu_{\infty,i}=\mu_i$ and $\sigma_{\infty,i}^2 = 0$, while for $i\in \overline{\mathcal{I}}$, $\sigma_{\infty,i}^2 > 0$. 
We let 
\[ 
\Pi_{\infty} = N(\mu_{\infty,1}, \sigma_{\infty,1}^2) \otimes N(\mu_{\infty,2}, \sigma_{\infty,2}^2) \otimes \cdots \otimes N(\mu_{\infty,k}, \sigma_{\infty,k}^2),
\]
and for each $i\in A$, we define
\[
\alpha_{\infty,i} \triangleq  \Prob_{\theta \sim \Pi_\infty}\left( \theta_i > \max_{j\neq i} \theta_j\right).
\]
For $i\in{\overline{\mathcal{I}}}$ is nonempty, we have $\alpha_{\infty,i}\in(0,1)$ since $\sigma_{\infty,i}^2>0$. This implies $\alpha_{\infty,1} < 1$ and so 
\[
\lim_{n\to \infty} -\frac{1}{n}\log(1-\alpha_{n, 1}) = \lim_{n\to\infty} -\frac{1}{n}\log(1-\alpha_{\infty,1})=0.
\]

Now suppose $\overline{\mathcal{I}}$ is empty. By definition, $\alpha_{n,1} = \Prob_{\theta \sim \Pi_n}\left( \theta_1 > \max_{i\neq 1} \theta_i\right)$, so $1 - \alpha_{n,1} = \Prob_{\theta \sim \Pi_n} \left(\cup_{i\neq 1} (\theta_i\geq \theta_1)\right)$, and then we have 
\begin{equation} \label{bound on alpha}
\max_{i\neq 1} \Prob_{\theta \sim \Pi_n} \left(\theta_i\geq \theta_1\right) \leq 1-\alpha_{n,1}  \leq \sum_{i\neq 1 } \Prob_{\theta \sim \Pi_n} \left(\theta_i\geq \theta_1\right)\leq (k-1)\max_{i\neq 1} \Prob_{\theta \sim \Pi_n} \left(\theta_i\geq \theta_1\right)
\end{equation}
where the second inequality uses the union bound. 

To simplify the presentation, we need to introduce the following asymptotic notation. We say two real-valued sequences $\{a_n\}$ and $\{b_n\}$ are \textsl{logarithmically equivalent} if $\lim_{n\to\infty}1/n\log(a_n/b_n) = 0$. We denote this by $a_n\doteq b_n$. Using equation \ref{bound on alpha}, we conclude
\[
1-\alpha_{n,1}\doteq \max_{i\neq 1}\Prob_{\theta \sim \Pi_n} \left(\theta_i\geq \theta_1\right).
\]
Next we want to show that for $i\neq 1$, $\Prob_{\theta \sim \Pi_n}\left(\theta_i\geq \theta_1\right)\doteq \exp\left(\frac{-(\mu_{n,i}-\mu_{n,1})^2}{2\sigma^2(1/T_{n,i}+1/T_{n,1})}\right)$. 
Note that at time $n$, $\theta_i-\theta_1\sim N(\mu_{n,i}-\mu_{n,1},\sigma_{n,i}^2+\sigma_{n,1}^2)$ and $\sigma_{n,i}^2+\sigma_{n,1}^2=\sigma^2(1/T_{n,i}+1/T_{n,1})$. Since every arm is sampled infinite times, when $n$ is large, $\mu_{n,1} \geq \mu_{n,i}$, and then using Lemma \ref{tail}, we have
\[
\frac{1}{\sqrt{2\pi}} \exp\left(\frac{-\left(\sqrt{\sigma_{n,i}^2+\sigma_{n,1}^2}+ \mu_{n,1}-\mu_{n,i}\right)^2}{2(\sigma_{n,i}^2+\sigma_{n,1}^2)}\right)
\leq \Prob_{\theta \sim \Pi_n}\left(\theta_i - \theta_1\geq 0\right) 
\leq \frac{1}{2} \exp\left(\frac{-(\mu_{n,1}-\mu_{n,i})^2}{2(\sigma_{n,i}^2+\sigma_{n,1}^2)}\right),
\]
which implies
\[
\frac{1}{n}\log\left(\frac{1}{\sqrt{2\pi}}\right)-\frac{1}{2n}-\frac{\mu_{n,1}-\mu_{n,i}}{n\sqrt{\sigma_{n,i}^2+\sigma_{n,1}^2}}
\leq
\frac{1}{n}\log\left(\frac{\Prob_{\theta \sim \Pi_n}\left(\theta_i \geq  \theta_1\right) }{\exp\left(\frac{-(\mu_{n,1}-\mu_{n,i})^2}{2(\sigma_{n,i}^2+\sigma_{n,1}^2)}\right)} \right)
\leq \frac{1}{n}\log\left(\frac{1}{2}\right).
\]
Note that when $\mu_{n,1}\geq \mu_{n,i}$,
\[
0\leq \frac{\mu_{n,1}-\mu_{n,i}}{n\sqrt{\sigma_{n,i}^2+\sigma_{n,1}^2}}=\frac{\mu_{n,1}-\mu_{n,i}}{\sigma\sqrt{n(n/T_{n,i}+n/T_{n,1})}} \leq \frac{\mu_{n,1}-\mu_{n,i}}{\sigma\sqrt{2n}}
\]
where the last equality uses $T_{n,i},T_{n,1}<n$. Using the squeeze theorem, we have 
\[
\lim_{n\to\infty} \frac{\mu_{n,1}-\mu_{n,i}}{n\sqrt{\sigma_{n,i}^2+\sigma_{n,1}^2}} = 0,
\]
and
\[
\lim_{n\to\infty} \frac{1}{n}\log\left(\frac{\Prob_{\theta \sim \Pi_n}\left(\theta_i \geq \theta_1 \right) }{\exp\left(\frac{-(\mu_{n,1}-\mu_{n,i})^2}{2(\sigma_{n,i}^2+\sigma_{n,1}^2)}\right)} \right) = 0.
\]
Hence, $\Prob_{\theta \sim \Pi_n}\left(\theta_i\geq \theta_1\right)\doteq \exp\left(\frac{-(\mu_{n,i}-\mu_{n,1})^2}{2\sigma^2(1/T_{n,i}+1/T_{n,1})}\right)$.
Then we have 
\begin{align*}
    1 - \alpha_{n,i} &\doteq \max_{i\neq 1}\Prob_{\theta \sim \Pi_n} \left(\theta_i\geq \theta_1\right) \\
    &\doteq  \max_{i\neq 1} \left\{\exp\left(\frac{-(\mu_{n,i}-\mu_{n,1})^2}{2\sigma^2(1/T_{n,i}+1/T_{n,1})}\right) \right\}\\
    &\doteq \exp\left(-n\min_{i\neq 1} \left\{\frac{(\mu_{n,i}-\mu_{n,1})^2}{2\sigma^2(n/T_{n,i}+n/T_{n,1})} \right\} \right)
\end{align*}
where the second equality uses the property that if $a_{n,i}\doteq b_{n,i}$ for each $i=1,\ldots,c$ where $c$ a positive integer, then $\max_{i\in\{1,\ldots,c\}} a_{n,i} \doteq \max_{i\in\{1,\ldots,c\}} b_{n,i}$.

Let $W\triangleq \left\{w = (w_1,\ldots,w_k) \,:\, \sum_{i=1}^k w_i = 1 \text{ and } w_i\geq 0,\forall i\in A \right\}$ denote the set of possible proportions on $k$ arms. \citet{russo2016simple} showed that 
\[
\Gamma^* = \max_{w\in W} \min_{i\neq 1} \frac{(\mu_{i}-\mu_{1})^2}{2\sigma^2(1/w_i+1/w_1)},
\]
and given $\beta\in(0,1)$,
\[
\Gamma^*_\beta = \max_{w\in W: w_1=\beta} \min_{i\neq 1} \frac{(\mu_{i}-\mu_{1})^2}{2\sigma^2(1/w_i+1/w_1)}.
\]
Under any sampling rule,
\begin{align*}
1 - \alpha_{n,i} &\doteq \exp\left(-n\min_{i\neq 1} \left\{\frac{(\mu_{n,i}-\mu_{n,1})^2}{2\sigma^2(n/T_{n,i}+n/T_{n,1})} \right\} \right) \\
&\geq \exp\left(-n\max_{w\in W}\min_{i\neq 1} \left\{\frac{(\mu_{n,i}-\mu_{n,1})^2}{2\sigma^2(1/w_i+1/w_1)} \right\} \right)
\end{align*}
Since every arm is sampled infinite times, as $n\to\infty$, $\mu_{n,i}\to \mu_i$ and $\mu_{n,1}\to \mu_1$, and thus
\[
\limsup_{n\to\infty} -\frac{1}{n}\log(1-\alpha_{n,i}) \leq \Gamma^*.
\]
If $T_{n,i}/n \to w^*_i$ for each $i\in A$, then for each $i\neq 1$, we have
\[
\lim_{n\to\infty} \frac{(\mu_{n,i}-\mu_{n,1})^2}{2\sigma^2(n/T_{n,i}+n/T_{n,1})} = \frac{(\mu_{i}-\mu_{1})^2}{2\sigma^2(1/w^*_i + 1/\beta)} = \Gamma^*,
\]
and thus 
\[
1 - \alpha_{n,i} \doteq \exp\left(-n\min_{i\neq 1} \left\{\frac{(\mu_{n,i}-\mu_{n,1})^2}{2\sigma^2(n/T_{n,i}+n/T_{n,1})} \right\} \right) \\
\doteq \exp\left(-n\Gamma^* \right),
\]
which implies
\[
\lim_{n\to\infty} -\frac{1}{n}\log(1-\alpha_{n,i}) = \Gamma^*.
\]
Similarly, for $\beta\in(0,1)$, under any sampling rule satisfying $T_{n,1}/n\to\beta$, we have 
\[
\limsup_{n\to\infty} -\frac{1}{n}\log(1-\alpha_{n,i}) \leq \Gamma_\beta^*,
\]
and under any sampling rule satisfying $T_{n,i}/n\to w_i^\beta$ for each $i\in A$,
\[
\lim_{n\to\infty} -\frac{1}{n}\log(1-\alpha_{n,i}) = \Gamma_\beta^*.
\]

\section{Proof of Theorem \ref{thm: sc2}} \label{sufficient condition 2}

%The statement is true for $\beta = 1$ since  $\Gamma_1^*=0$.

Let $\beta\in(0,1)$. Recall that TTEI begins with an improper prior for each arm $i\in A$ with $\mu_{1,i}=0$ and $\sigma_{1,i} = \infty$, so for any $i\in A$ and $n\in\mathbb{N}$, $\mu_{n,i} = \hat{\mu}_{n,i}$, i.e., the posterior mean equals the empirical mean, and thus $I_n^*=\argmax_{i\in A} \mu_{n,i}$ is identical to $\hat{I}_n^*=\argmax_{i\in A} \hat{\mu}_{n,i}$. We can rewrite $Z_n$ used in the Chernoff's stopping rule as follows,
\[
Z_n = \min_{j\in A\setminus\{I_n^*\}} Z_{n,I_n^*,j}
\] 
where the Generalized Likelihood Ratio statistic is
\[
Z_{n,I_n^*,j} = T_{n,I_n^*}d(\mu_{n,I_n^*},\mu_{n,I_n^*,j}) + T_{n,j}d(\mu_{n,j},\mu_{n,I_n^*,j})
\]
where
\[
\mu_{n,I_n^*,j} = \frac{T_{n,I_n^*}}{T_{n,I_n^*}+T_{n,j}}\mu_{n,I_n^*} + \frac{T_{n,j}}{T_{n,I_n^*}+T_{n,j}}\mu_{n,j}.
\]
Note that $\Delta_{\min} = \min_{i\neq j}|\mu_i-\mu_j|>0$. Then by definition of $T_\beta^{\Delta_{\min}/4}$, for all $i\in A$ and $n\geq T_\beta^{\Delta_{\min}/4}$, $|\mu_{n,i}-\mu_i|\leq \Delta_{\min}/4$, which implies $\mu_{n,1}>\ldots\mu_{n,k}$, and thus $I_n^*=1$. Using $d(x,y)=(x-y)^2/(2\sigma^2)$, for $n\geq T_\beta^{\Delta_{\min}/4}$, we have
\[
\frac{Z_n}{n} = \min_{i\in A\setminus\{1\}} \frac{(\mu_{n,i}-\mu_{n,1})^2}{2\sigma^2(n/T_{n,i}+n/T_{n,1})}.
\]
Note that 
\[
\Gamma_\beta^* = \frac{(\mu_{2}-\mu_{1})^2}{2\sigma^2\left(1/w_{2}^\beta+1/\beta\right)} =\ldots = \frac{(\mu_{k}-\mu_{1})^2}{2\sigma^2\left(1/w^\beta_{k}+1/\beta\right)}
\]
and when $\beta\in(0,1)$, $\Gamma_\beta^*>0$. Given $\epsilon > 0$, there exists $\epsilon'\in \left(0,\Delta_{\min}/4\right]$ such that for all $n\geq N^\epsilon \triangleq T_\beta^{\epsilon'}$, $|\mu_{n,i}-\mu_i|\leq \epsilon'$ and $|T_{n,i}/n-w^\beta_i|\leq \epsilon', \forall i\in A$ can imply $Z_n/n\geq \Gamma_\beta^*-\epsilon$. We have $\mathbb{E}\left[N^\epsilon\right] = \mathbb{E}\left[T_\beta^{\epsilon'}\right]<\infty$.

Let $\delta\in(0,1)$ and $\alpha > 0$. By Proposition \ref{GK}, the stopping time $\tau_{\delta} = \inf\left\{n\in\mathbb{N} \,:\, Z_n > \log(C n^\alpha/\delta) \right\}$ can ensure
$\mathbb{P}\left(\tau_\delta < \infty, \argmax_{i\in A}\mu_{\tau_\delta,i}\neq 1\right) \leq \delta.$ 

For $\epsilon\in \left(0,\Gamma_\beta^*/(1+\alpha)\right)$, when $n\geq N^\epsilon$, $Z_n\geq (\Gamma_\beta^*-\epsilon)n>0$. Let $M^\epsilon \triangleq \left\lceil\max\{N^\epsilon,1/\epsilon^2\}\right\rceil$ where the ceil function $\lceil x \rceil$ is the least integer greater than or equal to $x$.  
Now let us consider the following two cases.
\begin{enumerate}
	\item $\exists r\in\left[1,M^\epsilon\right]$ such that $Z_r > \log(C r^\alpha/\delta)$\\
	This case implies $\tau_\delta \leq M^\epsilon$.
	
	\item $\forall r\in[1,M^\epsilon]$, $Z_r \leq \log(C r^\alpha/\delta)$\\
	This case implies $\tau_\delta \geq M^\epsilon + 1$. Note that $M^\epsilon = \left\lceil\max\{N^\epsilon,1/\epsilon^2\}\right\rceil\geq N^\epsilon$, so for $n\geq M^\epsilon$,  $Z_n \geq (\Gamma_\beta^*-\epsilon)n$.
	Let $x^\epsilon$ be the solution of $(\Gamma_\beta^*-\epsilon)x = \log(Cx^\alpha/\delta)$. Since $(\Gamma^*_\beta-\epsilon)M^\epsilon \leq Z_{M^\epsilon} \leq \log(C(M^\epsilon)^\alpha/\delta)$, we have $x^\epsilon\geq M^\epsilon$, which implies $x^\epsilon \geq 1/\epsilon^2$, and then $\log(x^\epsilon)\leq (x^\epsilon)^{1/2}\leq \epsilon x^\epsilon$. Hence,
	$
	(\Gamma_\beta^*-\epsilon)x^\epsilon = \log(C(x^\epsilon)^\alpha/\delta)\leq \log(C) + \alpha\epsilon x^\epsilon + \log(1/\delta),
	$
	which implies
	\[
	x^\epsilon \leq \frac{\log(C)+\log(1/\delta)}{
		\Gamma_\beta^*-(1+\alpha)\epsilon}.
	\]
	Let $L_\delta^\epsilon \triangleq \inf\left\{  n\geq M^\epsilon \,:\,(\Gamma_\beta^*-\epsilon)n > \log(C n^\alpha/\delta) \right\}$. Since $(\Gamma_\beta^*-\epsilon)x^\epsilon = \log(C(x^\epsilon)^\alpha/\delta)$, we have 
	\[
	L_\delta^\epsilon \leq \lceil x^\epsilon\rceil +1 \leq \left\lceil  \frac{\log(C)+\log(1/\delta)}{
		\Gamma_\beta^*-(1+\alpha)\epsilon} \right\rceil +1 < \frac{\log(C)+\log(1/\delta)}{
		\Gamma_\beta^*-(1+\alpha)\epsilon} +2.
	\]
	We notice that $Z_{L_\delta^\epsilon} \geq (\tau_\beta^*-\epsilon)L_\delta^\epsilon > \log(C(L_\delta^\epsilon)^\alpha/\delta)$, so we have $\tau_\delta\leq L_\delta^\epsilon$.
\end{enumerate}
Combining the above two cases, we have $\tau_\delta\leq M^\epsilon + L_\delta^\epsilon$, and thus $\mathbb{E}[\tau_\delta]\leq \mathbb{E}[M^\epsilon] + \mathbb{E}[L_\delta^\epsilon]$. Note that $M^\epsilon = \left\lceil\max\{N^\epsilon,1/\epsilon^2\}\right\rceil$ and $\mathbb{E}[N^\epsilon] < \infty$ imply $\mathbb{E}[M^\epsilon] < \infty$. 

Now we fix $\tilde{\epsilon} = (\alpha-1)\Gamma_\beta^*/[\alpha(1+\alpha)] \in \left(0,\Gamma_\beta^*/(1+\alpha)\right)$, then we have
\[
L_\delta^{\tilde{\epsilon}} < \frac{\log(C)+\log(1/\delta)}{\Gamma_\beta^*-(1+\alpha)\epsilon} +2 = \alpha\left[\frac{\log(C)+\log(1/\delta)}{\Gamma_\beta^*}\right]+2 = \left[\frac{\alpha\log(C)}{\Gamma_\beta^*} + 2\right] + \frac{\alpha\log(1/\delta)}{\Gamma_\beta^*}.
\]
Therefore, we have
\[
\limsup_{\delta\to 0}\frac{\mathbb{E}[\tau_\delta]}{\log(1/\delta)}\leq  \limsup_{\delta\to 0}\frac{\mathbb{E}\left[M^{\tilde{\epsilon}}\right] + \mathbb{E}\left[L_\delta^{\tilde{\epsilon}}\right]}{\log(1/\delta)} \leq \frac{\alpha}{\Gamma_\beta^*}.
\]

\section{Preliminaries} \label{preliminaries}

In this section, we introduce several preliminary results which is used in the theoretical analysis of TTEI.

\subsection{Properties of $f(x) = x\Phi(x) + \phi(x)$}
We provide several properties of the function $f(x) = x\Phi(x) + \phi(x)$ including its monotonicity, upper bound and lower bound.
\begin{lemma}
\label{basic}
$f(x)$ is positive and increasing on $\mathbb{R}$.
\end{lemma}
\begin{proof}
This is true since $f'(x) = \Phi(x)\geq 0$ and $\lim_{x\to-\infty}f(x) = 0$.
\end{proof}

\begin{lemma}
\label{upper}
For $x>0$,
\[
f(-x) < \phi(-x).
\]
\end{lemma}
\begin{proof}
For $x>0$, $f(-x) = -x\Phi(-x)+\phi(-x) < \phi(-x)$.
\end{proof}

\begin{lemma}
\label{lower}
For $x\geq2$,
\[
f(-x) > \frac{1}{x^3}\phi(-x).
\]
\end{lemma}
\begin{proof}
Let  
$g(x) = \frac{1}{x} [f(-x) - \frac{1}{x^3}\phi(-x)]= -\Phi(-x)+\frac{1}{x}\phi(-x) - \frac{1}{x^4}\phi(-x)$. We have $g'(x) = (-x^{-2}+x^{-3}+4x^{-5})\phi(x) = x^{-5}(-x+2)(x^2+x+2)\phi(x)$, which implies that $g(x)$ is decreasing in $[2,\infty)$. We notice that $g(2)> 0$ and $\lim_{x\to\infty}g(x)=0$, so for $x\geq 2$, $g(x)>0$. Therefore, for $x\geq 2$, $f(-x) > \frac{1}{x^3}\phi(-x).$
\end{proof}
Lemmas \ref{upper} and \ref{lower} provides the upper and lower bounds for $f(\cdot)$, which is used to study the expected improvement measures.

\subsection{Maximal Inequalities}
In the theoretical analysis of TTEI, we need a bound on the difference between the empirical mean $\mu_{n,i}$ and the unknown true mean $\mu_i$ for each arm $i\in A$ at period $n$, and a bound on the difference between $T_{n,i}$ and $\Psi_{n,i}$, two measurements of effort allocated to arm $i$ up to period $n$. Two sample-path dependent variables $W_1$ and $W_2$ are required to obtain the two bounds.

\begin{lemma}
\label{W1}
Under any sampling rule beginning with an improper prior for each arm $i\in A$ with $\mu_{1,i}=0$ and $\sigma_{1,i} =\infty$, $\mathbb{E} [e^{\lambda W_1}] < \infty$ for all $\lambda>0$ where
\[
W_1 \triangleq \max_{n\in\mathbb{N}} \max_{i \in A}  \,\, \sqrt{\frac{T_{n,i}+1}{\log(e+T_{n,i})}}\left| \frac{\mu_{n,i} - \mu_i}{\sigma}  \right|.
\]	 
\end{lemma}
\begin{proof}

Under any sampling rule beginning with an improper prior for each arm $i\in A$ with $\sigma_{1,i} =\infty$ and $\mu_{1,i}=0$ for each arm $i\in A$, if $T_{n,i}=\sum_{\ell=1}^{n-1} \mathbf{1}\{I_\ell = i\} = 0$, $\mu_{n,i} =\mu_{1,i}=0$, and if $T_{n,i} > 0$,  
\[
\mu_{n,i} = \frac{1}{T_{n,i}} \sum_{\ell=1}^{n-1} 
\mathbf{1}\{I_\ell = i\}Y_{\ell,I_\ell}.
\]
A mathematically equivalent way of simulating the system is to generate a collection of independent variables $(X_{n,i})_{n\in \mathbb{N}, i\in A}$ where each $X_{n,i}\sim N(\mu_i, \sigma^2)$. At time $n$, the algorithm selects an arm $I_n$, and observes the real valued response $X_{S_{n,I_n}, I_n}$ where $S_{n,I_n} \triangleq \sum_{\ell=1}^{n} \mathbf{1}\{I_\ell = i\}$. 
%Here $X_{t,i}$ is the observation of arm $i$ on the $t$th time we sample the arm. 
For all $i\in A$, we let $\overline{X}_{0,i} = 0$, and for $n\in \mathbb{N}$, $\overline{X}_{n,i}= \frac{1}{n}\sum_{\ell=1}^{n} X_{\ell,i} $ denote the empirical mean of arm $i$ up to the $n$th time it is chosen. We will bound
\[
\widetilde{W} \triangleq \max_{n \in \mathbb{N}\cup\{0\}} \max_{i \in A}  \,\, \sqrt{\frac{n+1}{\log(e+n)}}\left| \frac{\overline{X}_{n,i} - \mu_i}{\sigma}  \right|.
\]
%Note that for arm $i\in A$, when $T_{n,i}=\sum_{\ell=1}^{n-1} \mathbf{1}\{I_\ell = i\} = 0$, then
%$
%\sqrt{(T_{n,i}+1)/\log(e+T_{n,i})}\left| (\mu_{n,i} - \mu_i)/\sigma \right| = \mu_i/\sigma
%$
%due to $\mu_{n,i}=0$.
When every arm is played infinitely often, $W_1=\widetilde{W}$. One always has $W_1 \leq \widetilde{W}$, so it is sufficient to bound $\mathbb{E}[ e^{\lambda\widetilde{W}}]$ for all $\lambda>0$. Notice that $\widetilde{W} = \max\{\xi,|\mu_1|/\sigma,\ldots,|\mu_k|/\sigma\}\leq\xi+ \sigma^{-1}\sum_{i\in A} |\mu_i| $
where
\[
\xi \triangleq \max_{n \in \mathbb{N}} \max_{i \in A}  \,\, \sqrt{\frac{n+1}{\log(e+n)}}\left| \frac{\overline{X}_{n,i} - \mu_i}{\sigma}  \right|.
\]
Hence, it suffices to bound $\mathbb{E}[ e^{\lambda \xi}]$ for all $\lambda>0$.

For all $n\in \mathbb{N}$ and $i\in A$, we define $Z_{n,i} \triangleq \sqrt{n}\left( \frac{\overline{X}_{n,i} -\mu_i}{ \sigma} \right)$, and then
\[
\xi = \max_{n \in \mathbb{N}} \max_{i \in A} \sqrt{\frac{n+1}{n\log(e+n)}}|Z_{n,i}|. 
\]
Each $Z_{n,i}\sim N(0,1)$, and thus by Lemma \ref{tail}, $Z_{n,i}$ satisfies the tail bound $\mathbb{P}(|Z_{n,i}|\geq z) \leq e^{-z^2/2}$ for $z >0$. Therefore, for all $x\geq2$ 
\begin{eqnarray*}
\mathbb{P}\left( \xi \geq 2x   \right) 
&=& \mathbb{P}\left( \exists n\in \mathbb{N}, i\in A \, : \,   |Z_{n,i}| \geq 2\sqrt{\frac{n\log(e+n)}{n+1}}  x  \right) \\
&\leq & \sum_{n,i} \mathbb{P}\left(|Z_{n,i}| \geq 2\sqrt{\frac{n\log(e+n)}{n+1}}  x \right) \\
& \leq &  \sum_{n,i} \exp\left(-\frac{2n\log(e+n)}{n+1} x^2\right) \\
&=& k\sum_{n} \exp\left(-\frac{2n\log(e+n)}{n+1} x^2\right) \\
&\overset{(*)}{\leq}& k\sum_{n} \exp\left(-2\log(e+n)-\frac{n}{n+1}x^2\right)\\
&=& k\sum_{n} \left(\frac{1}{e+n}\right)^2 e^{-\frac{n}{n+1}x^2}\\
&\leq& C e^{-x^2/2}.
\end{eqnarray*}
where step $(*)$ uses the $ab\geq a+b$ when $a,b\geq 2$ and $C= k \sum_{n\in \mathbb{N}} (e+n)^{-2}< \infty$ is a constant. Then for all $\lambda>0$, 
\[ 
\mathbb{E} \left[e^{\lambda \xi}\right] = \intop_{x=1}^{\infty} \mathbb{P}\left(e^{\lambda \xi } \geq x\right) dx \overset{(*)}{=} \intop_{u=0}^{\infty} \mathbb{P}\left(e^{\lambda \xi} \geq e^{2\lambda u} \right) 2\lambda e^{2\lambda u } du  \leq 2+ C\intop_{u=2}^{\infty} e^{-u^2/2}\cdot 2\lambda e^{2\lambda u} du < \infty 
\]
where in step $(*)$, we have substituted $x= e^{2\lambda u}$. Hence, for all $\lambda>0$, $\mathbb{E} \left[e^{\lambda W_1}\right] < \infty.$
\end{proof}

%Note that under TTEI, every arm is played infinitely often, so $W=\xi$.
This result provides a bound for the difference between the empirical mean of an arm and its true unknown mean. For $i\in A$ and $n \in \mathbb{N}$
\[
|\mu_{n,i} - \mu_i| \leq \sigma W_1 \sqrt{\frac{\log(e+T_{n,i})}{T_{n,i}+1}}.
\]

Then we introduce the second sample-path dependent variable $W_2$, and the following lemma on the difference between two measurements of effort under any top-two sampling rule, which at each time, measures one of the two designs that appear most promising given current evidence.

\begin{lemma}
\label{W2}
Under any top-two sampling rule with parameter $\beta\in(0,1)$ beginning with an improper prior for each arm $i\in A$ with $\mu_{1,i}=0$ and $\sigma_{1,i} =\infty$, $\mathbb{E} [e^{\lambda W_2}] < \infty$ for all $\lambda>0$ where
\[ 
W_2 \triangleq \max_{n \in \mathbb{N}} \max_{i \in A} \,\,  \frac{|T_{n,i} - \Psi_{n,i}|}{\sqrt{ \left(1+ \Psi_{n,i}/\beta_{\min} \right)\log\left( e^2+\Psi_{n,i}/\beta_{\min}  \right)}}.  
\]
\end{lemma}

\begin{proof}
Similar to the proof for Lemma \ref{W1}, it suffices to show $\mathbb{P}(W_2\geq x)\leq ke^{-x^2/2}$ for all $x\geq 2$.

Fix some $i \in A$. Define for each $n\in \mathbb{N}$
\[
D_n \triangleq  T_{n,i}-\Psi_{n,i} = \sum_{\ell=1}^{n-1} d_\ell
\]
where 
\[ 
d_n \triangleq \mathbf{1}(I_n=i) - \psi_{n,i} = \mathbf{1}(I_n=i) - \mathbb{P}(I_n = i | \mathcal{F}_{n-1}).
\] 
Then $\mathbb{E}[d_n | \mathcal{F}_{n-1}] = 0$ and  $D_n$ is a zero mean martingale. Now, note $\psi_{n,i} \in \{0, \beta, 1-\beta \}$ almost surely, and set 
\[ 
X_n := \mathbf{1}( \psi_{n,i} >0)
\]
to be the indicator that $i$ is among the top-two in period $n$. We can see that $d_n = X_n d_n$, and so 

\[
D_n = \sum_{\ell=1}^{n-1} X_\ell d_\ell. 
\] 
Here $\{X_n \}$ is a binary valued previsable process (i.e. $X_n$ is $\mathcal{F}_{n-1}$ measureable), and $d_n$ is a zero-mean $\mathcal{F}_{n}$ adapted process with increments bounded as $|d_n| \leq 1$ almost surely.

The quadratic variation of $D_n$ is 
\[ 
\langle D \rangle_n = \sum_{\ell=1}^{n-1} \mathbb{E}[X_{\ell} d_{\ell}^2 | \mathcal{F}_{\ell-1} ]= \sum_{\ell=1}^{n-1}X_{\ell} \beta(1-\beta) 
\]
and so the magnitude of fluctuations of the martingale $D_n$ scale with the number of times $i$ is in the top-two. 

There are a number of martingale analogues to the central limit theorem, which suggest that $D_n =O_{P}\left(\sqrt{\langle  D \rangle_n}\right)$. To establish this formally, we apply the theorem of self-normalized martingale processes \cite{pena2008self}, which bound processes like $D_{n} / \sqrt{\langle D \rangle_n}$. We will apply a result established in \cite{abbasi2012online}. 
	
Because $|d_n| \leq 1$, applying Hoeffding's Lemma implies 
\[ 
E[e^{\lambda d_n } | \mathcal{F}_{n-1}] \leq e^{\lambda^2 / 2}, \qquad \lambda \in \mathbb{R}
\]
and so $d_n$ is 1-sub--Gaussian conditioned on $\mathcal{F}_{n-1}$. Applying Corollary 8 of \cite{abbasi2012online} implies that for any $\delta>0$, with probability least $1-\delta$
\[
|D_n| \leq \sqrt{ 2 \left(1+ \sum_{\ell=1}^{n-1} X_\ell\right)\log\left( \frac{\sqrt{1+\sum_{\ell=1}^{n-1}X_\ell}}{\delta} \right)  }, \qquad \forall n\in \mathbb{N}
\]
Analogously, for any $x\geq 2$ with probability at least $1-e^{-x^2/2}$,
\begin{eqnarray*}
|D_n| &\leq& \sqrt{ 2 \left(1+ \sum_{\ell=1}^{n-1} X_\ell\right)\log\left( \frac{\sqrt{1+\sum_{\ell=1}^{n-1}X_\ell} }{e^{-x^2/2}} \right)  } \\
&=& \sqrt{ \left(1+ \sum_{\ell=1}^{n-1} X_\ell\right)\left(\log\left( 1+\sum_{\ell=1}^{n-1}X_\ell  \right) + x^2 \right) } \\
&\leq& \sqrt{ \left(1+ \sum_{\ell=1}^{n-1} X_\ell\right)\left(\log\left( e^2+\sum_{\ell=1}^{n-1}X_\ell  \right) + x^2 \right) }\\
&\leq& \sqrt{ \left(1+ \sum_{\ell=1}^{n-1} X_\ell\right)\log\left( e^2+\sum_{\ell=1}^{n-1}X_\ell  \right)x^2 }\\
\end{eqnarray*}
for all $n\in \mathbb{N}$, where the last step uses that $ab \geq a+b$ for $a,b\geq 2$. Then, for all $x\geq 2$
\[
\mathbb{P}\left( \max_{n\in \mathbb{N}} \frac{|D_n|}{\sqrt{ \left(1+ \sum_{\ell=1}^{n-1} X_\ell \right)\log\left( e^2+\sum_{\ell=1}^{n-1}X_\ell  \right)}} \geq x\right) \leq e^{-x^2/2}  
\]
Since $\Psi_{n,i} \geq \beta_{\min} \sum_{\ell=1}^{n-1} X_\ell$, we have shown that for any $i$, 
\[
\mathbb{P}\left( \max_{n\in \mathbb{N}} \frac{|T_{n,i} - \Psi_{n,i}|}{\sqrt{ \left(1+ \Psi_{n,i}/\beta_{\min} \right)\log\left( e^2+\Psi_{n,i}/\beta_{\min}  \right)}} \geq x\right) \leq e^{-x^2/2}  
\]
Taking a union bound over $i\in A$ implies $\mathbb{P}( W_2 \geq x) \leq k e^{-x^2/2}$ for any $x\geq 2$. 
\end{proof}

This result implies that for any period $n$ and arm $i$,
\[
|T_{n,i} - \Psi_{n,i}| \leq   W_2\sqrt{ \left(1+ \Psi_{n,i}/\beta_{\min} \right)  \log\left( e^2+\Psi_{n,i}/\beta_{\min}  \right)}. 
\]

%Next we provide an upper bound on the difference between $T_{n,i}$ and $\Psi_{n,i}$.
The next result provides another bound, which is used in the theoretical analysis of TTEI.

\begin{lemma}
\label{111}
Under TTEI with parameter $\beta\in(0,1)$ beginning with an improper prior for each arm $i\in A$ with $\mu_{1,i}=0$ and $\sigma_{1,i} =\infty$, for all $n\in \mathbb{N}$ and arm $i\in A$,
\[
|T_{n,i} - \Psi_{n,i}| <  \left(2 + \frac{3\Psi_{n,i}^{3/4}}{\beta_{\min}}\right)W_2.
\]
\end{lemma}
\begin{proof}
Fix some arm $i\in A$. If arm $i$ is never chosen in either case 1 or case 2 of TTEI up to period $n$, then $\Psi_{n,i} = 0$, and thus
\[
|T_{n,i} - \Psi_{n,i}| \leq   W_2\sqrt{ \left(1+ \Psi_{n,i}/\beta_{\min} \right)  \log\left( e^2+\Psi_{n,i}/\beta_{\min}  \right)} < 2W_2
\]
Once arm $i$ has been chosen in either case 1 or case 2 of TTEI, $\Psi_{n,i} \geq \beta_{\min}$. Then we have $1+\Psi_{n,i}/\beta_{\min} < 3\Psi_{n,i}/\beta_{\min}$ and $\log\left( e^2+\Psi_{n,i}/\beta_{\min}  \right)<3(\Psi_{n,i}/\beta_{\min})^{1/2}$, which leads to
\[
|T_{n,i} - \Psi_{n,i}| 
< 3W_2(\Psi_{n,i}/\beta_{\min})^{3/4} < \frac{3\Psi_{n,i}^{3/4}}{\beta_{\min}}W_2.
\]
Hence, 
\[
|T_{n,i} - \Psi_{n,i}| < \max\left\{2, \frac{3\Psi_{n,i}^{3/4}}{\beta_{\min}}\right\}W_2 < \left(2 + \frac{3\Psi_{n,i}^{3/4}}{\beta_{\min}}\right)W_2.
\]
\end{proof}

\subsection{Technical Lemmas}
The following technical lemma is used to quantify the time after which TTEI satisfies a certain property. We want to write such a time as a polynomial of sample-path dependent variables. 

\begin{lemma}
\label{tech}
Fix constants $c_0>c_1> 0$ and $c,c_2> 0$. Then for any $a_1,a_2 > 0$, there exists a $X = \poly(a_1,a_2)$ such that for all $x\geq X$,
\[
\exp\left(cx^{c_0}-a_1x^{c_1}\right) > a_2x^{c_2}.
\]
\end{lemma}
\begin{proof}
There exists $X_1 = \poly(a_1)$ such that for all $x\geq X_1$, $cx^{c_0-c_1}-a_1 > 1$. In addition, there exists $X_2 = \poly(a_2)$ such that for all $x\geq X_2$, $\exp\left(x^{c_1}\right) > a_2x^{c_2}$. Hence, for all $x\geq X\triangleq \max\{X_1,X_2\}$,
\[
\exp\left(cx^{c_0}-a_1x^{c_1}\right) = \exp\left(x^{c_1}\left(cx^{c_0-c_1}-a_1\right)\right)\geq \exp\left(x^{c_1}\right) > a_2x^{c_2}.
\]
\end{proof}

\section{Results specific to TTEI} \label{sampling proportions}
In this section, we present theoretical results specific to the proposed TTEI policy. The main challenge is ensuring $\mathbb{E}[T_\beta^\epsilon]$ is finite where $T_\beta^\epsilon$ is the time after which for each arm, its empirical mean and empirical proportion are $\epsilon$-accurate. To do this, we present several results for any sample path (up to a set of measure zero), and show that $T_\beta^\epsilon$ depends at most polynomially on $W_1$ and $W_2$. By Lemmas \ref{W1} and \ref{W2}, the expected value of polynomials of $W_1$ and $W_2$ is finite. This ensures that $\mathbb{E}[T_\beta^\epsilon]$ is finite, which immediately establishes that TTEI achieves the sufficient conditions for both notions of optimality.

\subsection{Sufficient Exploration}
We first show that every arm is sampled frequently under TTEI.

\begin{proposition}
\label{c1}
Under TTEI with parameter $\beta\in(0,1)$, there exists $N_1=\poly(W_1,W_2)$ such that for all $n\geq N_1$,
\[
T_{n,i} \geq \sqrt{n/k}, \qquad \forall i\in A.
\]
\end{proposition}

To prove this proposition, we first need to define two under-sampled sets for all $L>0$ and $n\in\mathbb{N}$:
\[
U_n^L \triangleq \{i\in A \,:\, T_{n,i} < L^{1/2}\}
\]
and
\[
V_n^L \triangleq \{i\in A \,:\, T_{n,i} < L^{3/4}\}.\footnote{We fix the exponent here to be $3/4$. Indeed, it can be changed to $1/2+\epsilon$ for any $\epsilon>0$. We just need a gap between the exponent here and $1/2$ in $U_n^L$.}
\]
Let $\overline{U_n^L} \triangleq A\setminus U_n^L$ and $\overline{V_n^L} \triangleq A\setminus V_n^L$. Then Proposition \ref{c1} can be proved using the following two lemmas.
%The following result 
Note that in this paper, $X = \poly(W_1,W_2)$ means that $X=\mathcal{O}(W_1^{c_1}W_2^{c_2})$ for positive constants $c_1$ and $c_2$ where $(\sigma,k,\mu_1,\ldots,\mu_k,\beta)$ are treated as constants throughout the proof.

\begin{lemma}
\label{main}
Under TTEI with parameter $\beta\in(0,1)$, there exists $L_1 = \poly(W_1)$ such that for all $L\geq L_1$ and $n\leq kL$,\footnote{$L$ could be any value, but $n$ must be integer value.} if $U_{n}^L$ is nonempty, then $I_{n}^{(1)}\in V_{n}^L$ or $I_{n}^{(2)}\in V_{n}^L$.
\end{lemma}
\begin{proof}
%Let $I_{n+1}^{*} = \argmax_j \mu_{n,j}$. 

First of all, we will show that if $I_n^{(1)}\in \overline{V_{n}^L}$, then $I_{n}^*\in \overline{V_{n}^L}$ where $I_n^*=\argmax_{i\in A}\mu_{n,i}$. We prove this by contradiction. Suppose $I_{n}^*\in V_{n}^L$. By definition, $T_{n,I_{n}^{(1)}} > T_{n,I_{n}^*}$, which implies $\sigma_{n,I_{n}^{(1)}} < \sigma_{n,I_{n}^*}$. By Lemma \ref{basic}, we have
\[
v^{(1)}_{n,I_{n}^{(1)}}
=\sigma_{n,I_{n}^{(1)}} f\left( \frac{\mu_{n,I_{n}^{(1)}}-\mu_{n,I_{n}^*}}{\sigma_{n,I_{n}^{(1)}}} \right) <\sigma_{n,I_{n}^*} f( 0 ) 
=v^{(1)}_{n,I_{n}^{*}},
\]
which contradicts the definition of $I_{n}^{(1)}$. Hence, if $I_n^{(1)}\in \overline{V_{n}^L}$, then $I_{n}^*\in \overline{V_{n}^L}$.

Secondly we will show that when $L$ is sufficiently large, if $I_n^{*}\in\overline{V_{n}^L}$, then for all $i\in \overline{V_{n}^L}\setminus\{I_n^*\}$, $\mu_{n,i} - \mu_{n,I_n^*}\leq -0.5\Delta_{\min}$ where $\Delta_{\min} = \min_{i\neq j}|\mu_i-\mu_j|>0$.
%Suppose $I_{n}^{(1)}\neq I_{n}^*$. This implies the cardinality $\left|\overline{V_{n}^T}\right|\geq 2$. 
By Lemma \ref{W1}, for all $i\in\overline{V_{n}^L}$,
\[
|\mu_{n,i}-\mu_{i}|\leq \sigma W_1\sqrt{\frac{\log(e+T_{n,i})}{T_{n,i}+1}}\leq \sigma W_1 \sqrt{\frac{\log(e+ L^{3/4})}{L^{3/4}+1}}
\]
where the last inequality is valid because $g(x)=\log(e+x)/(x+1)$ is positive and decreasing on $(0,\infty)$ and $T_{n,i}\geq L^{3/4}$. Note that for $L\geq 1$, $\log(e+L^{3/4})\leq 2L^{1/4}$. Then there exists $M_1=\poly(W_1)$ such that for all $L\geq M_1$,
\[
\sqrt{\frac{\log(e+ L^{3/4})}{L^{3/4}+1}} \leq \sqrt{\frac{2L^{1/4}}{L^{3/4}+1}}\leq \frac{\Delta_{\min}}{4\sigma W_1}.
\]
%Let $\overline{V_{n}^L}\setminus\{I_n^*\}$ be nonempty. 
Suppose there exists $\tilde{i}\in\overline{V_{n}^L}\setminus\{I_n^*\}$ such that $\mu_{\tilde{i}}>\mu_{I_n^*}$. Then for $L\geq M_1$, we have
\begin{align*}
\mu_{n,\tilde{i}} - \mu_{n,I_{n}^*}
\geq& \mu_{\tilde{i}} - \sigma W_1 \sqrt{\frac{\log(e+L^{3/4})}{L_{3/4}+1}} - \mu_{I_{n}^*} - \sigma W_1 \sqrt{\frac{\log(e+L^{3/4})}{L_{3/4}+1}}     \\
=& (\mu_{\tilde{i}} - \mu_{I_{n}^*}) - 2\sigma W_1 \sqrt{\frac{\log(e+L^{3/4})}{L^{3/4}+1}}\\
\geq& \Delta_{\min} - 2\sigma W_1(\Delta_{\min}/4\sigma W_1) = 0.5\Delta_{\min},
\end{align*}
which contradicts the definition of $I_n^*$. Hence, for $L\geq M_1$, if $I_n^*\in\overline{V_{n}^L}$, then $\mu_{I_n^*}>\mu_i$ for all $i\in\overline{V_{n}^L}\setminus\{I_n^*\}$ (note that we assume that all arm-means are unique), and thus
\[
\mu_{n,i} - \mu_{n,I_{n}^*} \leq (\mu_i - \mu_{I_{n}^*}) + 2\sigma W_1 \sqrt{\frac{\log(e+ L^{3/4})}{L^{3/4}+1}} \leq -\Delta_{\min} + 0.5\Delta_{\min} = -0.5\Delta_{\min}.
\]

Thirdly we will show when $L$ is sufficiently large and $n\leq kL$, if $I_n^{(1)}\in\overline{V_{n}^L}$ (which implies $I_n^{*}\in\overline{V_{n}^L}$), then $v^{(1)}_{n,I_n^*}>v^{(1)}_{n,i}$ for all $i\in\overline{V_{n}^L}\setminus\{I_n^*\}$, which implies $I_n^{(1)}=I_n^*$. For all $i\in\overline{V_{n}^L}\setminus\{I_n^*\}$, $\sigma_{n,i}^2 = \sigma^2/T_{n,i}\leq\sigma^2/L^{3/4}$, and when $L\geq M_1$, $\mu_{n,i} - \mu_{n,I_n^*}\leq -0.5\Delta_{\min}$, which lead to
\begin{equation}
\label{1}
v^{(1)}_{n,i}=\sigma_{n,i} f\left( \frac{\mu_{n,i}-\mu_{n,I_n^*}}{\sigma_{n,i}} \right) \leq \frac{\sigma}{L^{3/8}}f\left(\frac{-\Delta_{\min} L^{3/8}}{2\sigma}\right) < \frac{\sigma}{L^{3/8}}\phi\left(\frac{-\Delta_{\min} L^{3/8}}{2\sigma}\right)
\end{equation}
where the last inequality uses Lemma \ref{upper}. 
On the other hand, 
\begin{equation}
\label{2}
v^{(1)}_{n,I_{n}^*}=\sigma_{n,I_{n}^*} f(0) \geq \frac{\sigma}{ (kL)^{1/2}}\phi(0).
\end{equation}
There exists $M_2$ such that for all $L\geq M_2$, the right hand side of (\ref{2}) is larger than the right hand of (\ref{1}). Hence, for $L\geq\max\{M_1,M_2\}$ and $n\leq kL$, if $I_n^{(1)}\in\overline{V_{n}^L}$ (which implies $I_n^{*}\in\overline{V_{n}^L}$), then $v^{(1)}_{n,I_n^*}>v^{(1)}_{n,i}$ for all $i\in\overline{V_{n}^L}\setminus\{I_n^*\}$, which implies $I_n^{(1)}=I_n^*$.

Finally we will show that when $L$ is sufficiently large and $n\leq kL$, if $U_{n}^L$ is nonempty (which implies $V_{n}^L$ is nonempty by definition) and $I_{n}^{(1)}\in \overline{V_{n}^L}$ (which implies $I_{n}^{*}\in \overline{V_{n}^L}$), then $I_{n}^{(2)}\in V_{n}^L$. We have proved that for $L\geq\{M_1,M_2\}$, $I_n^{(1)}=I_n^*$. Then for all $i\in \overline{V_{n}^L}\setminus \{I_{n}^{*}\}$, 
\[
\mu_{n,i} - \mu_{n,I_{n}^{(1)}}  = \mu_{n,i} - \mu_{n,I_{n}^*} \leq -0.5\Delta_{\min},
\]
and by definition, 
\[
\sigma_{n,i}^2 + \sigma_{n,I_{n}^{(1)}}^2
=\sigma_{n,i}^2 + \sigma_{n,I_{n}^*}^2  =\frac{\sigma^2}{T_{n,i}} + \frac{\sigma^2}{T_{n,I_{n}^*}} \leq \frac{\sigma^2}{L^{3/4}} + \frac{\sigma^2}{L^{3/4}} < \frac{4\sigma^2}{L^{3/4}},
\]
which leads to 
\begin{equation}
\label{3}
v^{(2)}_{n,i} <  \frac{2\sigma}{L^{3/8}} f\left(\frac{-\Delta_{\min} L^{3/8}}{4\sigma}\right) < \frac{2\sigma}{L^{3/8}} \phi\left(\frac{-\Delta_{\min} L^{3/8}}{4\sigma}\right).
\end{equation}
where the last inequality uses Lemma \ref{upper}. On the other hand, for all $j\in U_{n}^L$,
\begin{align*}
\mu_{n,j} - \mu_{n,I_{n}^{(1)}} =& \mu_{n,j} - \mu_{n,I_{n}^*} \\
\geq& \mu_j - \sigma W_1 \sqrt{\frac{\log(e+T_{n,j})}{T_{n,j}+1}}- \mu_{I_{n}^*} - \sigma W_1 \sqrt{\frac{\log(e+T_{n,I_{n}^*})}{T_{n,I_{n}^*}+1}} \\
\geq& ( \mu_j - \mu_{I_{n}^*}) - 2\sigma W_1 \sqrt{\frac{\log(e)}{1}} = ( \mu_j - \mu_{I_{n}^*}) - 2\sigma W_1
\end{align*}
where the last inequality is valid because $g(x)=\log(e+x)/(x+1)$ is positive and decreasing on $(0,\infty)$ and $T_{n,j},T_{n,I_n^*}\geq 0$.
If $\mu_{I_n^*} > \mu_j$, $\mu_{n,j} - \mu_{n,I_{n}^{(1)}} \geq -\Delta_{\max} - 2\sigma W_1$ where $\Delta_{\max}=\max_{i,j\in A}(\mu_i-\mu_j)$; otherwise, $\mu_{n,j} - \mu_{n,I_{n}^{(1)}} \geq \Delta_{\min} - 2\sigma W_1 > -\Delta_{\max} - 2\sigma W_1$. Hence, we have $\mu_{n,j} - \mu_{n,I_{n}^{(1)}} \geq -\Delta_{\max} - 2\sigma W_1$, and by definition,
\[
\sigma_{n,j}^2 + \sigma_{n,I_{n}^{(1)}}^2 =\sigma_{n,j}^2 + \sigma_{n,I_{n}^*}^2 = \frac{\sigma^2}{T_{n,j}} + \frac{\sigma^2}{T_{n,I_{n}^*}} > \frac{\sigma^2}{L^{1/2}} + \frac{\sigma^2}{kL} > \frac{\sigma^2}{L^{1/2}},
\]
which leads to 
\[
v^{(2)}_{n,j} > \frac{\sigma}{L^{1/4}} f\left(\frac{-(\Delta_{\max} + 2\sigma W_1)L^{1/4}}{\sigma}\right).
\]
Let $M_3\triangleq (2\sigma/\Delta_{\max})^4$. Since $W_1\geq 0$ by definition, for all $L\geq M_3$, $(\Delta_{\max} + 2\sigma W_1)L^{1/4}/\sigma\geq2$, and then by Lemma \ref{lower}, we have 
\begin{equation}
\label{4}
v^{(2)}_{n,j} > \frac{\sigma}{L^{1/4}} f\left(\frac{-(\Delta_{\max} + 2\sigma W_1)L^{1/4}}{\sigma}\right) > \frac{\sigma^4}{L(\Delta_{\max} + 2\sigma W_1)^3 } \phi\left(\frac{-(\Delta_{\max} + 2\sigma W_1)L^{1/4}}{\sigma}\right).
\end{equation}
By Lemma \ref{tech}, there exists $M_4$ such that for all $L\geq M_4$, the right hand side of (\ref{4}) is larger than the right hand side of (\ref{3}). Therefore, for $L\geq L_1\triangleq \max\{M_1,M_2,M_3,M_4\}$ and $n\leq kL$, if $U_{n}^L$ is nonempty (which implies $V_{n}^L$ is nonempty by definition) and $I_{n}^{(1)}\in \overline{V_{n}^L}$ (which implies $I_{n}^{*}\in \overline{V_{n}^L}$), then $v^{(2)}_{n,j}>v^{(2)}_{n,i}$ for all $j\in U_n^L$ and $i\in \overline{V_n^L}$ (here we use $v^{(2)}_{n,I_n^*}=v^{(2)}_{n,I_n^{(1)}}=0$), which implies $I_n^{(2)}\notin \overline{V_n^L}$, and thus $I_n^{(2)}\in V_n^L$.

\end{proof}

Note that the floor function $\lfloor x\rfloor$ is the greatest integer less than or equal to $x$. Then based on Lemma \ref{main}, we have the following result.

\begin{lemma}
\label{empty}
Under TTEI with parameter $\beta\in(0,1)$, there exists $L_2 = \poly(W_1,W_2)$ such that for all $L\geq L_2$, $U_{\lfloor kL \rfloor}^L$ is empty.
\end{lemma}

\begin{proof}
There exists $M_1=\poly(W_2)$ such that for all $L\geq M_1$, we have $\lfloor L \rfloor - 1\geq k L^{3/4}$ and 
\[
\beta_{\min}\lfloor L\rfloor - 4kW_2 -\frac{6k\lfloor kL\rfloor^{3/4}}{\beta_{\min}}W_2 \geq kL^{3/4}
\]
where $\beta_{\min}=\min\{\beta,1-\beta\}>0$.
Let $L_2\triangleq\max\{L_1,M_1\}$ where $L_1=\poly(W_1)$ has been introduced in Lemma \ref{main}. Now We want to prove this statement by contradiction. 

Suppose there exists some $L\geq L_2$ such that $U_{\lfloor kL \rfloor}^L$ is nonempty. 
Then all $U_1^L,U_2^L,\ldots,U_{\lfloor kL\rfloor-1}^L,U_{\lfloor kL\rfloor}^L$ are nonempty, and thus by definition, all $V_1^L,V_2^L,\ldots,V_{\lfloor kL\rfloor-1}^L,V_{\lfloor kL\rfloor}^L$ are empty.
Since $L\geq L_2$, we have $\lfloor L \rfloor - 1\geq k L^{3/4}$, so at least one arm is measured at least $L^{3/4}$ times before period $\lfloor L\rfloor$, and thus $\left|V_{\lfloor L \rfloor}^L\right|\leq k-1$.

Now we want to prove $\left|V_{\lfloor 2L\rfloor}^L\right| \leq k-2$. For all $\ell = \lfloor L\rfloor,\lfloor L\rfloor+1,\ldots,\lfloor 2L\rfloor-1$, $U_\ell^L$ is nonempty, then by Lemma \ref{main}, we have $I_n^{(1)}\in V_\ell^L$ or $I_n^{(2)}\in V_\ell^L$, and thus $\sum_{i\in V_\ell^L} \psi_{l,i}=\sum_{i\in V_\ell^L}\mathbb{P}(I_\ell = i |  \mathcal{F}_{\ell-1})\geq \beta_{\min}$, which implies $\sum_{i\in V_{\lfloor L\rfloor}^L} \psi_{l,i}\geq \beta_{\min}$ due to $V_\ell^L\subseteq V_{\lfloor L\rfloor}^L$. Hence, we have 
\[
\sum_{i\in V_{\lfloor L\rfloor}^L} \left(\Psi_{{\lfloor 2L\rfloor},i} -  \Psi_{{\lfloor L\rfloor},i}\right)
=\sum_{\ell=\lfloor L \rfloor}^{\lfloor 2L\rfloor-1}
\sum_{i\in V_{\lfloor L\rfloor}^L} \psi_{\ell,i}
\geq \beta_{\min}{\lfloor L\rfloor}
\]
where the inequality uses the fact that $\lfloor a+b \rfloor\geq \lfloor a \rfloor + \lfloor b \rfloor$ for $a,b\geq 0$.
Then by Lemma \ref{111}, we have
\begin{align*}
&\sum_{i\in V_{\lfloor L\rfloor}^L} \left(T_{{\lfloor 2L\rfloor},i} -  T_{{\lfloor L\rfloor},i}\right) \\
\geq& \sum_{i\in V_{\lfloor L\rfloor}^L} \left(\Psi_{{\lfloor 2L\rfloor},i} -  \Psi_{{\lfloor L\rfloor},i}\right) - 
\sum_{i\in V_{\lfloor L\rfloor}^L} \left[\left(2 + \frac{3\Psi_{\lfloor 2L\rfloor,i}^{3/4}}{\beta_{\min}}\right)W_2  + \left(2 + \frac{3\Psi_{\lfloor L\rfloor,i}^{3/4}}{\beta_{\min}}\right)W_2 \right]\\
\geq& \beta_{\min}\lfloor L \rfloor - 2\sum_{i\in V_{\lfloor L\rfloor}^L} \left(2 + \frac{3\Psi_{\lfloor kL\rfloor,i}^{3/4}}{\beta_{\min}}\right)W_2   \\
>& \beta_{\min}\lfloor L\rfloor - 2k \left(2+\frac{3\Psi_{\lfloor kL\rfloor,i}^{3/4}}{\beta_{\min}}\right)W_2\\
>& \beta_{\min}\lfloor L\rfloor - 4kW_2 -\frac{6k\lfloor kL\rfloor^{3/4}}{\beta_{\min}}W_2 \geq kL^{3/4}
\end{align*}
where the second last inequality uses that for all $i\in A$ and $n\in\mathbb{N}$, $\Psi_{n,i}\leq \beta_{\max}(n-1)<n$, and the last inequality is valid because of the construction of $L_2$ and $L\geq L_2$.
Hence, at least one arm in $V_{\lfloor L\rfloor}^L$ is measured at least $L^{3/4}$ times in periods $\left[\lfloor L\rfloor,\lfloor 2L\rfloor\right)$, and thus $\left|V_{\lfloor 2L\rfloor}^L\right| \leq k-2$.

Similarly, we can prove that for $r = 3,\ldots,k$, at least one arm in $V_{\lfloor (r-1)L\rfloor}^L$ is measured at least $L^{3/4}$ times in periods $\left[\lfloor (r-1)L\rfloor,\lfloor rL\rfloor\right)$, so $\left|V_{\lfloor rL\rfloor}^L\right|\leq k-r$. Hence, $\left|V_{\lfloor kL\rfloor}^L\right|=0$, i.e., $V_{\lfloor kL\rfloor}^L$ is empty,  which implies that $U_{\lfloor kL\rfloor}^L$ is empty.
%This implies that $V_{3T,l}$ should be empty, which leads to a contradiction.
\end{proof}

Now we can prove Proposition \ref{c1}.

\paragraph{Proof of Proposition \ref{c1}.}
Let $N_1=kL_2$ where $L_2=\poly(W_1,W_2)$ introduced in Lemma \ref{empty}. For all $n\geq N_1$, we let $L=n/k$, then by Lemma \ref{empty}, we have $U_{\lfloor kL \rfloor}^L=U_n^{n/k}$ is empty, which by definition results in that for all $i\in A$, $T_{n,i} \geq \sqrt{n/k}$.

\subsection{Concentration of Empirical Means}

When $n$ is large, using the bound on the difference between the empirical mean $\mu_{n,i}$ and the unknown true mean $\mu_i$  in terms of $T_{n,i}$ for each arm $i\in A$, we can formally show the concentration of $\mu_{n,i}$ to $\mu_i$ under TTEI.

\begin{proposition}
\label{all_mean}
Let $\epsilon > 0$. Under TTEI with parameter $\beta\in(0,1)$, there exists $N_2^\epsilon = \poly(W_1,W_2,1/\epsilon)$ such that for all $n\geq N_2^\epsilon$,
\[
\left|\mu_{n,i} - \mu_i \right|\leq \epsilon, \qquad \forall i\in A.
\]
\end{proposition}
\begin{proof}
By Lemma \ref{W1}, for all $i\in A$ and $n \in \mathbb{N}$,
\[
|\mu_{n,i} - \mu_i| \leq \sigma W_1 \sqrt{\frac{\log(e+T_{n,i})}{T_{n,i}+1}}.
\]
By Proposition \ref{c1}, for all $n\geq N_1$, for all $i\in A$, $T_{n,i}\geq \sqrt{n/k}$ , and thus
\[
|\mu_{n,i} - \mu_i| \leq \sigma W_1 \sqrt{\frac{\log(e+T_{n,i})}{T_{n,i}+1}} \leq \sigma W_1 \sqrt{\frac{\log(e+(n/k)^{1/2})}{(n/k)^{1/2}+1}}
\]
where the last inequality uses $g(x)=\log(e+x)/(x+1)$ is positive and decreasing on $(0,\infty)$.  Note that for $n\geq k$, $\log(e+(n/k)^{1/2})\leq 2(n/k)^{1/4}$. Then there exists $M_1^\epsilon=\poly(W_1,1/\epsilon)$ such that for all $n\geq M_1^\epsilon$,
\[
\sqrt{\frac{\log(e+ (n/k)^{1/2})}{(n/k)^{1/2}+1}} \leq \sqrt{\frac{2(n/k)^{1/4}}{(n/k)^{1/2}+1}}\leq \frac{\epsilon}{\sigma W_1}.
\]
Then for all $i\in A$ and $n\geq N_2^\epsilon \triangleq \max\{N_1,k,M_1^\epsilon\}$ where $N_1 = \poly(W_1,W_2)$ introduced in Proposition \ref{c1},  we have $|\mu_{n,i}-\mu_i|\leq \sigma W_1 [\epsilon/(\sigma W_1)]=\epsilon$.
%given $\epsilon>0$, we can make $\left|\mu_{n,i} - \mu_i \right|\leq \epsilon$.
\end{proof}

Recall that we assume the unknown arm-means are unique and $\mu_1 > \mu_2 \ldots > \mu_k$. If we set $\epsilon$ to a very small value in Lemma \ref{all_mean}, when $n$ is large, the empirical means are order as the true means, i.e.,  $\mu_{n,1} > \mu_{n,2} \ldots > \mu_{n,k}$, which implies the arm with the largest empirical mean is arm 1. In addition, we show that when $n$ is large, the arm selected in case 1 of TTEI is also arm 1.

\begin{lemma}
\label{first_best}
Under TTEI with parameter $\beta\in(0,1)$, there exists $N_3 = \poly(W_1,W_2)$ such that for all $n\geq N_3$,
$I_n^{(1)}=I_n^* = 1$.
\end{lemma}
\begin{proof}
Let $M_1\triangleq N_2^{\Delta_{\min}/4}$. By Proposition \ref{all_mean}, for all $n\geq M_1$,
\[
\left|\mu_{n,i} - \mu_i \right|\leq \Delta_{\min}/4, \qquad \forall i\in A
\]
where $\Delta_{\min}=\min_{i\neq j}|\mu_i-\mu_j|>0$, which implies $\mu_{n,1}>\mu_{n,2}>\ldots>\mu_{n,k}$, and thus $I_n^* = 1$. 

Now for $n\geq M_1$ and $i\neq I_n^*$, we have 
\begin{align*}
    \mu_{n,I_n^*} - \mu_{n,i} &= \mu_{n,1} - \mu_{n,i} \\
    &\geq \mu_{1} - \Delta_{\min}/4 - \mu_{i} - \Delta_{\min}/4\\
    &=(\mu_{1} - \mu_i) -\Delta_{\min}/2\\
    &\geq\Delta_{\min} - \Delta_{\min}/2 = \Delta_{\min}/2.
\end{align*}
By Proposition \ref{c1}, for $n\geq N_1$, $T_{n,i}\geq \sqrt{n/k}$ for all $i\in A$. Hence, for $n\geq \max\{N_1,M_1\}$ and $i\neq I_n^*$, we have 
\begin{equation}
\label{5}
v^{(1)}_{n,i} =\sigma_{n,i} f\left( \frac{\mu_{n,i}-\mu_{n,I_n^*}}{\sigma_{n,i}} \right)\leq \frac{\sigma k^{1/4}}{n^{1/4}} f\left(\frac{-\Delta_{\min}n^{1/4}}{2\sigma k^{1/4}}\right) < \frac{\sigma k^{1/4}}{n^{1/4}} \phi\left(\frac{-\Delta_{\min}n^{1/4}}{2\sigma k^{1/4}}\right)
\end{equation}
where the two inequalities use Lemmas \ref{basic} and \ref{upper}, respectively. On the other hand,
\begin{equation}
\label{6}
v^{(1)}_{n,I_n^*} =\sigma_{n,I_n^*} f(0) = \sigma_{n,I_n^*} \phi(0) > \frac{\sigma}{n^{1/2}}\phi(0)
\end{equation}
where the inequality uses $T_{n,I_n^*}\leq n-1 < n$. There exists $M_2$ such that for all $n\geq M_2$, the right hand side of (\ref{6}) is larger than the right hand side of (\ref{5}). Hence, for all $n\geq N_3\triangleq\max\{N_1,M_2,M_2\}$, $v^{(1)}_{n,I_n^*} > v^{(1)}_{n,i}$ for all $i\neq I_n^*$, which implies $I_n^{(1)} = I_n^* = 1$.
\end{proof}

\subsection{Tracking the Asymptotic Proportion of the Best Arm}

In this subsection, we show that when the number of arm draws goes large, the empirical proportion for the best arm concentrates to the tuning parameter $\beta$ used in TTEI. 

\begin{lemma}
\label{best_Psi}
Let $\epsilon > 0$. Under TTEI with parameter $\beta\in(0,1)$, there exists $N_4^\epsilon = \poly(W_1,W_2,1/\epsilon)$ such that for all $n\geq N_4^\epsilon$,
\[
\left|\frac{\Psi_{n,1}}{n} -\beta \right|\leq \epsilon.
\]
%where $\beta_{(1)}$ is the asymptotic proportion allocated to the best arm $(1)$ and $\beta_{(1)} = \beta$.
\end{lemma}

\begin{proof}
By Lemma \ref{first_best}, for all $n\geq N_3$, we have $I_n^{(1)} = 1$. Then we have
\begin{align*}
\frac{\Psi_{n,1}}{n}
=& \frac{1}{n}\left(\sum_{\ell=1}^{N_3-1}\psi_{\ell,1}+\sum_{\ell=N_3}^{n-1}\psi_{\ell,1}\right) \\
\leq&  \frac{1}{n}\left[\beta_{\max}(N_3-1) + \beta(n-N_3)\right] \\
<& \beta + \frac{(\beta_{\max}-\beta)N_3}{n}
\end{align*}
where $\beta_{\max} = \max\{\beta,1-\beta\}$, and
\begin{align*}
\frac{\Psi_{n,1}}{n}
=& \frac{1}{n}\left(\sum_{\ell=1}^{N_3-1}\psi_{\ell,1}+\sum_{\ell=N_3}^{n-1}\psi_{\ell,1}\right) \\
\geq&  \frac{1}{n}\beta(n-N_3) \\
=& \beta - \frac{\beta N_3}{n}.
\end{align*}
For all $n\geq \beta_{\max}N_3/\epsilon$, we have $(\beta_{\max}-\beta)N_3/n < \epsilon$ and $-\beta N_3/n\geq -\epsilon$. Therefore, for all $n\geq N_4^\epsilon\triangleq\max\{N_3,\beta_{\max}N_3/\epsilon\}$, we have $\left|\Psi_{n,1}/n -\beta \right|\leq \epsilon$.
\end{proof}

Based on Lemma \ref{best_Psi}, we can prove the next result showing the concentration of $T_{n,1}/n$ to $\beta$.

\begin{lemma}
\label{best_T}
Let $\epsilon > 0$. Under TTEI with parameter $\beta\in(0,1)$, there exists $N_5^\epsilon = \poly(W_1,W_2,1/\epsilon)$ such that for all $n\geq N_5^\epsilon$,
\[
\left| \frac{T_{n,1}}{n} -\beta \right|  \leq \epsilon.
\]
\end{lemma}

\begin{proof}
It suffices to prove this statement for $\epsilon \in (0,\beta)$. By Lemma \ref{best_Psi}, for all $n\geq N_4^{\epsilon/2}$, $|\Psi_{n,1}/n - \beta| \leq \epsilon/2$, which implies $\Psi_{n,1} \geq (\beta-\epsilon/2)n$. Lemma \ref{111} implies that for all $n\geq M_1^\epsilon\triangleq \max\left\{N_4^{\epsilon/2},2/\beta\right\}$, 
\begin{equation}
\label{7}
\left|\frac{T_{n,1}}{\Psi_{n,1}}-1\right| 
%< \left(\frac{2}{\Psi_{n,1}} + \frac{3}{\beta_{\min}\Psi_{n,1}^{1/4}}\right)W_2 
\leq \left(\frac{2}{\Psi_{n,1}^{1/4}} + \frac{3}{\beta_{\min}\Psi_{n,1}^{1/4}}\right)W_2
\leq \frac{(2+3/\beta_{\min})W_2}{(\beta-\epsilon/2)^{1/4}n^{1/4}} < \frac{(2+3/\beta_{\min})W_2}{(\beta/2)^{1/4}n^{1/4}}
\end{equation}
where the second inequality is valid since $\Psi_{n,1} \geq (\beta-\epsilon/2)n > (\beta /2)n\geq 1$.
There exists $M_2^\epsilon=\poly(W_2,1/\epsilon)$ such that for all $n\geq M_2^\epsilon$, the right hand side of (\ref{7}) is less than $\epsilon/(2\beta+\epsilon)$.
Hence, for all $n\geq N_5^\epsilon \triangleq \max\left\{M_1^\epsilon,M_2^\epsilon\right\}$, $|T_{n,1}/\Psi_{n,1}-1|< \epsilon/(2\beta+\epsilon)$ and $|\Psi_{n,1}/n - \beta| \leq \epsilon/2$, and thus we have
\[
\frac{T_{n,1}}{n}  < \left(1+ \frac{\epsilon}{2\beta + \epsilon}\right) \frac{\Psi_{n,1}}{n}
\leq \left(1+ \frac{\epsilon}{2\beta + \epsilon}\right) (\beta+\epsilon/2)  = \beta + \epsilon
\]
and 
\[
\frac{T_{n,1}}{n} > \left(1 - \frac{\epsilon}{2\beta + \epsilon}\right) \frac{\Psi_{n,1}}{n}
\geq \left(1 - \frac{\epsilon}{2\beta + \epsilon}\right) (\beta - \epsilon/2) > \beta - \epsilon,
\]
which leads to $|T_{n,1}/n-\beta|<\epsilon$.
\end{proof}

\subsection{Tracking the Asymptotic Proportions of All Arms}

Besides the best arm, we can further show that for each arm, its empirical proportion concentrates to its optimal proportion when the number of arm draws goes large.

\begin{proposition}
\label{all_T}
Let $\epsilon > 0$. Under TTEI with parameter $\beta\in(0,1)$, there exists $N_7^\epsilon = \poly(W_1,W_2,1/\epsilon,\epsilon)$ such that for all $n\geq N_7^\epsilon$, 
\[
\left|\frac{T_{n,i}}{n} -w^\beta_i \right|\leq \epsilon, \qquad \forall i\in A.
\]
\end{proposition}

To prove this proposition, we need some further notations. For any $n\in \mathbb{N}$, we define the under-sampled set 
\[
P_n = \left\{i\neq 1 \,:\, \frac{T_{n,i}}{n} - w^\beta_i < 0 \right\},
\]
where the unique vector $\left(w^\beta_{2},\ldots,w^\beta_{k}\right)$ satisfies $\sum_{i=2}^{k} w^\beta_{i} = 1-\beta$ and  
\[
\frac{(\mu_{2}-\mu_{1})^2}{1/w^\beta_{2}+1/\beta} =\ldots = \frac{(\mu_{k}-\mu_{1})^2}{1/w^\beta_{k}+1/\beta}.
\] 
Then given $\epsilon > 0$, we define the over-sampled set
\[
O_n^\epsilon = \left\{i\neq 1 \,:\, \frac{T_{n,i}}{n} - w^\beta_i > \epsilon \right\}.
\]

The next result shows that when $n$ is large, the over-sampled set is empty. Based on this result, we can prove that when $n$ is large, the under-sampled set is also empty, which immediately establishes Proposition \ref{all_T}.
\begin{lemma}
\label{all_T_upper}

Let $\epsilon > 0$. Under TTEI with parameter $\beta\in(0,1)$, there exists $N_6^\epsilon = \poly(W_1,W_2,1/\epsilon,\epsilon)$ such that for all $n\geq N_6^\epsilon$, $O_n^\epsilon$ is empty.
\end{lemma}

\begin{proof}
If $O_n^{\epsilon/2}$ is empty, then $O_n^\epsilon$ is empty. Now let us consider the case that $O_n^{\epsilon/2}$ is nonempty, and it suffices to prove the statement for $\epsilon\in(0,\min\{\Delta_{\min}/2,1\})$. 

Fix $\epsilon\in(0,\min\{\Delta_{\min}/2,1\})$. For $\epsilon'\in(0,\epsilon/2)$, by Proposition \ref{all_mean} and Lemma \ref{best_T}, we can find large enough $M_1^{\epsilon'}=\poly(W_1,W_2,1/\epsilon')$ such that for all $n\geq M_1^{\epsilon'}$, both $|\mu_{n,i}-\mu_i|<\epsilon',\forall i\in A$ and $|T_{n,1}/n -\beta|\leq \epsilon'$ hold.

First we want to prove that for $n\geq M_1^{\epsilon'}$, if $O_n^{\epsilon/2}$ is nonempty, then $P_n$ is nonempty. We prove this by contradiction. Suppose $P_n$ is empty. Then for all $i\neq 1$, $T_{n,i}/n\geq w^\beta_i$. Since $O_n^{\epsilon/2}$ is nonempty, there exists some arm $\tilde{i}\neq 1$ such that $T_{n,\tilde{i}}/n > w^\beta_{\tilde{i}} + \epsilon/2$. In addition, for $n\geq M_1^{\epsilon'}$, $T_{n,1}/n \geq \beta -\epsilon'>\beta-\epsilon/2$. Hence,
\begin{align*}
\sum_{i\in A} T_{n,i}/n &= T_{n,1}/n + T_{n,\tilde{i}}/n + \sum_{i\neq 1, \tilde{i}} T_{n,i}/n\\
&> \beta - \epsilon/2 + w^\beta_{\tilde{i}} + \epsilon/2 + \sum_{i\neq 1, \tilde{i}} w^\beta_i\\
&=\sum_{i\in A} w^\beta_i = 1,
\end{align*}
which leads to a contradiction since $\sum_{i\in A} T_{n,i}/n = (n-1)/n<1$. Hence, for $n\geq M_1^{\epsilon'}$, if $O_n^{\epsilon/2}$ is nonempty, then $P_n$ is nonempty.

Next we will show that when $n$ is sufficiently large, $I_n^{(2)} \notin O_n^{\epsilon/2}$. By Lemma \ref{first_best}, for $n\geq N_3$, we have $I_n^{(1)} = I_n^* = 1$, and then for $i \neq 1$,
\[
v^{(2)}_{n,i} = \sqrt{\sigma_{n,i}^2+\sigma_{n,1}^2} f\left(\frac{\mu_{n,i}-\mu_{n,1}}{\sqrt{\sigma_{n,i}^2+\sigma_{n,1}^2}}\right)
\]
where $\sigma_{n,i}^2 = \sigma^2/T_{n,i}$ and $\sigma_{n,1}^2 = \sigma^2/T_{n,1}$.
%It suffices to only consider the case that $O_n^{\epsilon/2}$ is nonempty, which implies $P_n$ is nonempty for $n\geq M_1^{\epsilon'}$.
Note that for $n\geq M_1^{\epsilon'}$, $|\mu_{n,i}-\mu_i|<\epsilon',\forall i\in A$ and $|T_{n,1}/n -\beta|\leq \epsilon'$. Hence, for $n\geq \max\left\{N_3, M_1^{\epsilon'}\right\}$ and $i \in O_n^{\epsilon/2}$,
\[
v^{(2)}_{n,i} < \sigma\left(\frac{1}{w^\beta_i+\epsilon/2 }+\frac{1}{\beta-\epsilon'}\right)^{1/2}n^{-1/2}  \phi\left(\frac{(\mu_i-\mu_{1}+2\epsilon')n^{1/2}}{\sigma\left[1/(w^\beta_i+\epsilon/2)+1/(\beta-\epsilon')\right]^{1/2}}\right)
\]
where the inequality uses Lemma \ref{upper}. Note that $2\epsilon'<\epsilon <\Delta_{\min}/2$, so the value taken by $\phi(\cdot)$ is negative.
On the other hand, for $j\in P_n$,
\begin{align*}
v^{(2)}_{n,j} &> \sigma\left(\frac{1}{w^\beta_j }+\frac{1}{\beta + \epsilon'}\right)^{1/2} n^{-1/2}  f\left(\frac{(\mu_j-\mu_{1}-2\epsilon')n^{1/2}}{\sigma\left[1/w^\beta_j+1/(\beta+\epsilon')\right]^{1/2}}\right)\\
&>\sigma^4 \left(\frac{1}{w^\beta_j }+\frac{1}{\beta + \epsilon'}\right)^{2}(-\mu_j+\mu_{1}+2\epsilon')^{-3}n^{-2}
\phi\left(\frac{(\mu_j-\mu_{1}-2\epsilon')n^{1/2}}{\sigma\left[1/w^\beta_j+1/(\beta+\epsilon')\right]^{1/2}}\right)
\end{align*}
where the last inequality is valid by Lemma \ref{lower} since there exists $M_2^{\epsilon'}=\poly(1/\epsilon')$ such that for $n\geq M_2^{\epsilon'}$, the value taken by both $f(\cdot)$ and $\phi(\cdot)$ is less than $-2$. Let $M_3^{\epsilon'}\triangleq \max\left\{N_3,M_1^{\epsilon'},M_2^{\epsilon'}\right\}=\poly(W_1,W_2,1/\epsilon')$.
For any $i,j\in A$ such that $i\neq j$ and $i,j\neq 1$, we define the following constant in terms of $\epsilon$
\[
C^{\epsilon}_{i,j} \triangleq  \frac{(\mu_i-\mu_{1})^2}{1/(w^\beta_i+\epsilon/2)+1/\beta} - \frac{(\mu_j-\mu_{1})^2}{1/w^\beta_j+1/\beta}, 
\]
and we let 
\[
C^{\epsilon}_{\min} \triangleq \min_{\substack{i\neq j\\i,j\neq 1}} C^{\epsilon}_{i,j},
\]
and for $\epsilon'\in (0,\epsilon/2)$, we define the following function of $\epsilon'$
\[
g^\epsilon_{i,j}(\epsilon') \triangleq \frac{(\mu_i-\mu_{1}+2\epsilon')^2}{1/(w^\beta_i+\epsilon/2)+1/(\beta-\epsilon')} - 
\frac{(\mu_j-\mu_{1}-2\epsilon')^2}{1/w^\beta_j+1/(\beta+\epsilon')}.
\]
We know that
\[
\frac{(\mu_{2}-\mu_{1})^2}{1/w^\beta_{2}+1/\beta} =\ldots = \frac{(\mu_{k}-\mu_{1})^2}{1/w^\beta_{k}+1/\beta},
\]
so each $C^\epsilon_{i,j}>0$, and thus $C^\epsilon_{\min} > 0$.
Since each $g^\epsilon_{i,j}(\epsilon')$ is increasing as $\epsilon'$ is decreasing to 0, and $\lim_{\epsilon'\to 0}g^\epsilon_{i,j}(\epsilon')=C^\epsilon_{i,j} \geq C^\epsilon_{\min}$, there exists a threshold $\epsilon_{i,j}=\poly(\epsilon)\in(0,\epsilon/2)$ such that $g^\epsilon_{i,j}(\epsilon_{i,j}) \geq C^\epsilon_{\min}/2$ (note that $\epsilon<1$). We let
\[
\epsilon_{\min} \triangleq \min_{\substack{i\neq j\\i,j\neq 1}}\epsilon_{i,j}.
\]
Then for $n\geq M_3^{\epsilon_{\min}}$, for all $i\in O_n^{\epsilon/2}$ and $j\in P_n$,
\begin{equation}
\label{8}
\frac{v^{(2)}_{n,j}}{v^{(2)}_{n,i}} 
> D_{i,j}^\epsilon n^{-3/2} \exp\left(\frac{C^\epsilon_{\min} n}{4\sigma^2}\right)
\geq D^\epsilon_{\min} n^{-3/2} \exp\left(\frac{C^\epsilon_{\min} n}{4\sigma^2}\right),
\end{equation}
where 
\[
D_{i,j}^\epsilon \triangleq \frac{\sigma^4 \left(\frac{1}{w^\beta_j }+\frac{1}{\beta + \epsilon_{\min}}\right)^{2}(-\mu_j+\mu_{1}+2\epsilon_{\min})^{-3}}{ \sigma\left(\frac{1}{w^\beta_i+\epsilon/2 }+\frac{1}{\beta-\epsilon_{\min}}\right)^{1/2}}
\] 
and 
\[
D^\epsilon_{\min} \triangleq \min_{\substack{i\neq j\\i,j\neq 1}}D_{i,j}.
\]
%Note that $\epsilon_{\min}=\poly(\epsilon)$.  
%Since $\epsilon\in(0,\min\{\Delta_{\min}/2,1\})$ and $\epsilon'\in(0,\epsilon/2)$, 
Since $\epsilon_{\min}=\poly(\epsilon)$, there exists $M_4^\epsilon = \poly(1/\epsilon,\epsilon)$ such that for $n\geq M_4^\epsilon$, the right hand side of (\ref{8}) is greater than 1. Hence, for $n\geq M_5^\epsilon\triangleq \max\left\{M_3^{\epsilon_{\min}},M_4^\epsilon\right\}$ where $\epsilon_{\min}=\poly(\epsilon)$, we have $v^{(2)}_{n,j} > v^{(2)}_{n,i}$ for all $i\in O_n^{\epsilon/2}$ and $j\in P_n$, which implies $I_n^{(2)}\notin O_n^{\epsilon/2}$. Note that $M_5^{\epsilon}=\poly(W_1,W_2,1/\epsilon,\epsilon)$.
%Note that $M_3^{\epsilon_{\min}}=\poly(W_1,W_2,1/\epsilon_{\min})$ and $\epsilon_{\min}=\poly(\epsilon)$ imply $M_5^{\epsilon}=\poly(W_1,W_2,1/\epsilon,\epsilon)$.

Finally we will prove when $n$ is sufficiently large, $O_n^{\epsilon}$ is empty. Let $M^\epsilon \triangleq \max\left\{M_5^\epsilon,2/\epsilon\right\}$. There are two following cases on the set $O_{M^\epsilon}^{\epsilon/2}$.

\begin{enumerate}
    \item $\left|O_{M^\epsilon}^{\epsilon/2}\right| = 0$\\
    We will prove by induction that for all $n\geq M^\epsilon$, $O_n^{\epsilon}$ is empty. For $n = M^\epsilon$, $O_{n}^{\epsilon}$ is empty since $O_{n}^{\epsilon}\subseteq O_{n}^{\epsilon/2}$ and $O_{n}^{\epsilon/2}$ is empty.
    Now we suppose that $O_n^{\epsilon}$ is empty for some $n \geq M^\epsilon$, and we want to show that $O_{n+1}^{\epsilon}$ is empty.

    Note that $O_n^\epsilon$ is empty, and then only $I_n^{(1)}$ and $I_n^{(2)}$ may enter $O_{n+1}^\epsilon$.
    We known that for $n\geq M^\epsilon$, $I_n^{(1)} = 1$, which implies that $I_n^{(2)}\neq 1$ and only $I_n^{(2)}$ may enter $O_{n+1}^\epsilon$. In addition, for $n\geq M^\epsilon$, we have proved that $I_n^{(2)}\notin O_{n}^{\epsilon/2}$, which implies $T_{n,I_n^{(2)}}/n - w^\beta_{I_n^{(2)}} \leq \epsilon/2$. Since $n\geq M^\epsilon \geq 2/\epsilon$, $T_{n+1,I_n^{(2)}}/(n+1) - w^\beta_{I_n^{(2)}} \leq (T_{n,I_n^{(2)}}+1)/n - w^\beta_{I_n^{(2)}} \leq 1/n+\epsilon/2 \leq \epsilon$, which implies $I_n^{(2)}\notin O_{n+1}^{\epsilon}$, i.e., $I_n^{(2)}$ will not enter $O_{n+1}^{\epsilon}$. Hence, if $O_n^\epsilon$ is empty, then $O_{n+1}^\epsilon$ is empty.
    
    Therefore, by induction, for all $n\geq M^\epsilon$, $O_n^\epsilon$ is empty.

    %\cqcomment{stop here}
    \item $\left|O_{M^\epsilon}^{\epsilon/2}\right| \geq 1$\\    
    Similarly to the proof for case 1, we can show that for any arm $i\notin O_{M^\epsilon}^{\epsilon/2}$, it will not enter any $O_n^{\epsilon}$ for $n\geq M^{\epsilon}$. 
    
    Now let us consider arm $i\in O_{M^\epsilon}^{\epsilon/2}$.    
    Let $L^\epsilon_i$ be the time such that $i\in O_{n}^{\epsilon/2}$ for $n\in[M^\epsilon, L^\epsilon_i-1]$ and $i\notin O_{L^\epsilon_i}^{\epsilon/2}$. Similar to the proof for case 1, we can prove that for $i$ will not enter any $O_n^{\epsilon}$ for $n\geq L^\epsilon_i$.

    Let $M_6^\epsilon \triangleq \max_{i\in O_{M^\epsilon}^{\epsilon/2}} L^\epsilon_i$. For $n\geq M_6^\epsilon$, $O_n^\epsilon$ is empty. Note that $M_6^\epsilon=\poly(W_1,W_2,1/\epsilon,\epsilon)$.  
\end{enumerate}
Combining the above two cases, we conclude that there exists $N_6^\epsilon = \poly(W_1,W_2,1/\epsilon,\epsilon)$ such that for all $n\geq N_6^\epsilon$, $O_n^\epsilon$ is empty.

\end{proof}

Based on Lemma \ref{all_T_upper}, we can easily prove that when $n$ is large, the under-sampled set is also empty, which immediately establishes Proposition \ref{all_T}.

\paragraph{Proof of Proposition \ref{all_T}.}
Given $\epsilon>0$, by Lemmas \ref{best_T} and \ref{all_T_upper}, there exists $M_1^{\epsilon/k} = \poly(W_1,W_2,1/\epsilon,\epsilon)$ such that for $n\geq M_1^{\epsilon/k}$, $|T_{n,1}/n - w^\beta_1| \leq \epsilon/k$ where $w^\beta_1=\beta$ and $T_{n,i}/n - w^\beta_i \leq \epsilon/k$ for all $i\in A\setminus\{1\}$. Suppose there exists $i'\in A$ such that $T_{n,i'}/n - w^\beta_{i'} <  - \epsilon$. Then
\begin{align*}
\sum_{i\in A} T_{n,i}/n &= T_{n,i'}/n + \sum_{i\neq i'} T_{n,i}/n\\
&< w^\beta_{i'} - \epsilon + \sum_{i\neq i'} (w^\beta_i + \epsilon/k)\\
&=\sum_{i\in A} w^\beta_i + [-\epsilon+(k-1)\epsilon/k] \\ 
&= 1-\epsilon/k.
\end{align*}
On the other hand, for $n\geq k/\epsilon$, $\sum_{i\in A} T_{n,i}/n = (n-1)/n\geq 1-\epsilon/k$, which leads to a contradiction. Hence, for $n\geq N_7^{\epsilon} = \max\left\{M_1^{\epsilon/k},k/\epsilon\right\}$, for all $i\in A$, we have $-\epsilon \leq T_{n,i}/n - w^\beta_i \leq \epsilon/k$, which leads to $|T_{n,i}/n-w^\beta_i|<\epsilon$. Note that $N_7^\epsilon = \poly(W_1,W_2,1/\epsilon,\epsilon)$.

\subsection{Proof of Theorem \ref{thm: sampling proportions}}
For any $\epsilon>0$, by Propositions \ref{all_mean} and \ref{all_T}, for $n\geq N^\epsilon_\beta \triangleq \{N_2^\epsilon,N_7^\epsilon\}$, we have 
\[
|\mu_{n,i}-\mu_i|\leq\epsilon \quad\text{and}\quad |T_{n,i}/n - w^\beta_i|\leq\epsilon\qquad\forall i\in A.
\]
Note that $N_\beta^\epsilon=\poly(W_1,W_2,1/\epsilon,\epsilon)$. By Lemmas \ref{W1} and \ref{W2}, we have $\mathbb{E}[e^{\lambda W1}]<\infty$ and $\mathbb{E}[e^{\lambda W2}]<\infty$ for all $\lambda>0$, which implies that the expected value of any polynomial of $W_1$ and $W_2$ is finite, and thus $\mathbb{E}[N_\beta^\epsilon]<\infty$. By definition, $T_\beta^\epsilon \leq N_\beta^\epsilon$, so $\mathbb{E}[T_\beta^\epsilon]\leq\mathbb{E}[N_\beta^\epsilon]<\infty$.

Since $\epsilon$ can be arbitrary small, for any sample path (up to a set of measure zero), we have 
\[
\lim_{n \to \infty} \frac{T_{n,i}}{n} = w^{\beta}_{i} \qquad \forall i \in A.
\]

\end{document}